\documentclass[]{fairmeta}

\title{DP-RDM: Adapting Diffusion Models to Private Domains Without Fine-Tuning}

\author[1,2, 3]{Jonathan Lebensold}
\author[1]{Maziar Sanjabi}
\author[1]{Pietro Astolfi}
\author[1, 2, 3, 4]{Adriana Romero-Soriano}
\author[1]{Kamalika Chaudhuri}
\author[1,2]{Mike Rabbat}
\author[1]{Chuan Guo}

\affiliation[1]{Meta AI}
\affiliation[2]{Mila - Québec AI Institute, Montreal, Quebec, Canada }
\affiliation[3]{Department of Computer Science, McGill University, Montreal, Quebec, Canada}
\affiliation[4]{Canada CIFAR AI Chair}

\date{\today}
\correspondence{Jonathan Lebensold at \email{lebensold@meta.com}, Chuan Guo at \email{chuanguo@meta.com}}
\code{\url{https://github.com/facebookresearch/dp_rdm}}
\usepackage{wrapfig}
\usepackage{microtype}
\usepackage{graphicx}

\usepackage{hyperref}
\usepackage[hypcap=false]{caption}

\usepackage[algo2e]{algorithm2e}

\usepackage{amsmath}
\usepackage{amssymb, bbm}
\usepackage{bbold}
\usepackage{mathtools}
\usepackage{amsthm}
\usepackage{enumitem}
\usepackage{algorithm}
\usepackage{cleveref}

\theoremstyle{plain}
\newtheorem{theorem}{Theorem}[section]
\theoremstyle{definition}
\newtheorem{definition}[theorem]{Definition}

\usepackage{booktabs,makecell, multirow, tabularx}

\usepackage{capt-of,etoolbox}

\makeatletter
\apptocmd{\@maketitle}{\centering\teaserfig}{}{}
\makeatother

\newcommand{\eps}{\epsilon}

\newcommand{\R}{\mathcal{R}}

\renewcommand{\vec}[1]{\ensuremath{\mathbf{#1}}}

\def\R{\mathbb{R}}

\newcommand{\norm}[1]{\| #1 \|}

\newcommand{\DPRDMADAPT}{\textit{RDM-adapt DPR}}
\newcommand{\DPRDMFB}{\textit{RDM-fb DPR}}
\newcommand{\RDMFB}{\textit{RDM-fb PR}}
\newcommand{\RDMADAPT}{\textit{RDM-adapt PR}}

\newcommand\teaserfig{
\vspace{-.5em}
\begingroup
\begin{center}
\includegraphics[width=5in]{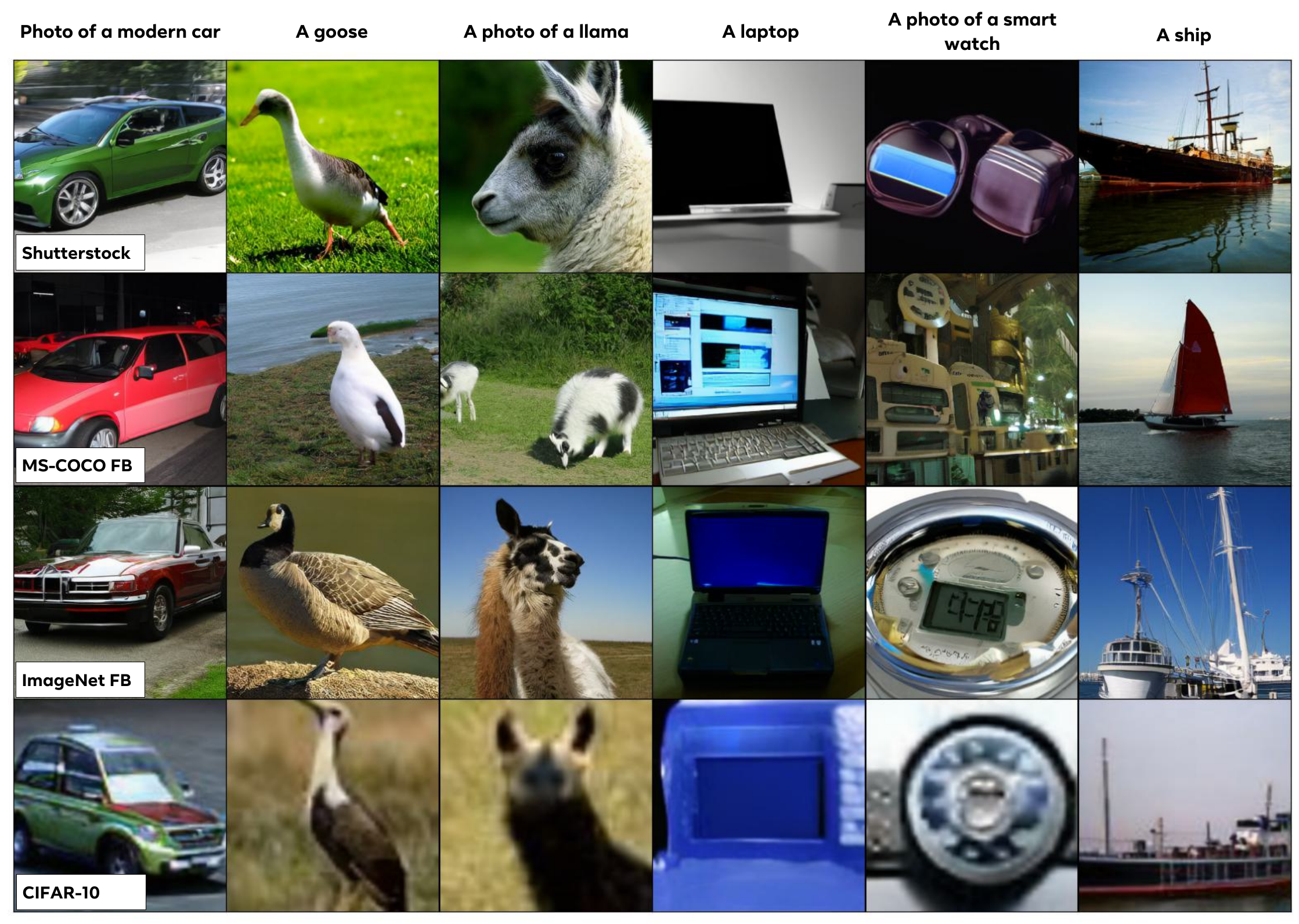}
\vspace{-2ex}
\end{center}
\label{fig:teaser}
\captionof{figure}{Samples generated by our \emph{differentially private retrieval-augmented diffusion model} (DP-RDM), which was trained on face-blurred ImageNet, using different private retrieval datasets at inference time: Shutterstock, MS-COCO with face-blurring (FB), ImageNet FB, and CIFAR-10. We calibrated the noise added in each row for a privacy budget of $\epsilon = 10$ after answering $1,000$ queries. Each query uses $k=18$ neighbors for retrieval augmentation and a 0.1\%-0.3\% random subset of the retrieval dataset. The differences between the generated images show how swapping the private retrieval dataset changes the distribution of the generated images to adapt to a given data domain, \emph{e.g.}, the differences in laptop and smart watch, or missing concepts such as goose, llama, smart watch and laptop in CIFAR-10. }\looseness-1
\endgroup}

\begin{document}
\abstract{
Text-to-image diffusion models have been shown to suffer from sample-level memorization, possibly reproducing near-perfect replica of images that they are trained on, which may be undesirable. To remedy this issue, we develop the first differentially private (DP) retrieval-augmented generation algorithm that is capable of generating high-quality image samples while providing provable privacy guarantees. Specifically, we assume access to a text-to-image diffusion model trained on a small amount of public data, and design a DP retrieval mechanism to augment the text prompt with samples retrieved from a private retrieval dataset. Our \emph{differentially private retrieval-augmented diffusion model} (DP-RDM) requires no fine-tuning on the retrieval dataset to adapt to another domain, and can use state-of-the-art generative models to generate high-quality image samples while satisfying rigorous DP guarantees.
For instance, when evaluated on MS-COCO, our DP-RDM can generate samples with a privacy budget of $\epsilon=10$, while providing a $3.5$ point improvement in FID compared to public-only retrieval for up to $10,000$ queries. 
}
\maketitle
\vspace{-0.1em}
\teaserfig

\section{Introduction}

Text-to-image diffusion models \citep{ho2020denoising, song2020denoising} enable highly customizable image generation, producing photo-realistic image samples that can be instructed through text prompting. However, recent studies also showed that it is possible to prompt text-to-image diffusion models to produce near-perfect replica of some of their training samples~\citep{somepalli2023diffusion, carlini2023extracting}.

One promising strategy for mitigating this risk is through \emph{differential privacy} (DP; \citet{dwork2006calibrating}). In essence, if a generative model satisfies DP, then its generated samples cannot critically depend on any individual training sample, therefore preventing the model from replicating any training image in a provable manner. Currently, the state-of-the-art in DP image generation focuses on adaptation through fine-tuning~\citep{ghalebikesabi2023differentially, lyu2023differentially}, where a diffusion model is first trained on public data (e.g., a licensed dataset) and then fine-tuned on a private dataset using DP-SGD~\citep{abadi2016deep}. This approach, although effective at a small scale, suffers immensely when scaling up the model and the private dataset due to the computational cost of DP fine-tuning \citep{sander2023tan, mehta2022large}. To date, most studies on DP image generation are limited to simple image datasets such as MNIST and CIFAR-10~\citep{dockhorn2022differentially}. 

To address this limitation, we adopt a different approach for private image generation through \emph{differentially private retrieval augmentation}. Retrieval-augmented generation (RAG; \citet{lewis2020retrieval}) is an emerging paradigm in generative modeling that performs generation in a semi-parametric manner. To generate a sample, the model utilizes both its learned parameters as well as a \emph{retrieval dataset} that consists of an arbitrary set of images. In effect, the model can adapt to different domains at generation time by changing the retrieval dataset without fine-tuning. The modularity of RAG makes it ideal for privacy-sensitive applications, where the sensitive data can be stored in the retrieval dataset for more fine-grained control of information leakage. 

We utilize this desirable property of RAG and define a \emph{differentially private retrieval-augmented diffusion model} (DP-RDM) that is capable of generating high-quality images based on text prompts while satisfying rigorous DP guarantees. To this end, we first demonstrate that without any privacy protection, retrieval-augmented diffusion models are prone to copying artifacts from samples in the retrieval dataset. Then, we propose a private retrieval mechanism that adds calibrated noise to the retrieved samples, along with modifications to an existing retrieval-augmented diffusion model architecture~\citep{blattmann2022retrieval} to accommodate this mechanism. As a result, our DP-RDM is capable of generating high-quality samples from complex image distributions at a moderate privacy cost of $\epsilon=10$; see Figure 1 for example. Notably, DP-RDM can potentially work with any state-of-the-art retrieval-augmented image generation model, regardless of the number of parameters and output resolution.
To the best of our knowledge, this is the first work to show privacy risks of using private data in the retrieval dataset for RAG and propose a DP solution for adapting text-to-image diffusion models to sensitive target domains. Our approach holds the potential of democratizing the adoption of such models in privacy-sensitive applications since it does not require any costly fine-tuning, nor does it require sharing sensitive data with a fine-tuning party.

\paragraph{Contributions.} Our main contributions are the following.
\begin{enumerate}[nosep, leftmargin=*]
    \item We demonstrate how retrieval-augmented diffusion models without any privacy protection can leak sample-level information from their retrieval dataset in the worst case.
    \item We propose a \emph{differentially private retrieval-augmented diffusion model} (DP-RDM) that provably mitigates this information leakage. DP-RDM utilizes a DP retrieval mechanism based on private $k$-NN~\citep{zhu2020private} and adapts the existing retrieval-augmented diffusion model architecture~\citep{blattmann2022retrieval} to this mechanism.
    \item We evaluate our DP-RDM on three datasets---CIFAR-10, MS-COCO and Shutterstock, a privately licensed dataset of 239M image-captions pairs---and show that it can effectively adapt to these datasets privately with minor loss in generation quality. On MS-COCO, our model is able to generate up to $10,000$ images at a privacy cost of $\epsilon=10$ while achieving an FID of $10.9$ (lower is better). In comparison, using only the public retrieval dataset yields an FID of $14.4$ with the same model.\looseness-1
\end{enumerate}

\section{Background}
We first present the necessary background on diffusion models, retrieval-augmented generation and differential privacy.\looseness-1

\subsection{Diffusion Models and Retrieval Augmented Generation}

\textbf{Diffusion models} \citep{ho2020denoising, song2020denoising} represent a leap of research progress in generative modeling yielding photorealistic generated images. During training, an image is corrupted with noise (\emph{i.e.}, forward process), and fed to a model (neural network), whose learning objective is to predict the noise in order to perform denoising (\emph{i.e.}, reverse process). At generation time, the model produces an image sample from noise through iterative denoising. Although this denoising process can be carried out in the raw pixel space, a more effective approach is to perform it in the latent space of a pre-trained autoencoder, giving rise to the so-called latent diffusion models (LDMs; \citet{rombach2022high}). Another popular extension to diffusion models is text-to-image diffusion models~\citep{rombach2022high, saharia2022photorealistic}, where the denoising process is conditioned on text captions from paired image-text datasets. In effect, the model learns to produce images that are relevant to a given text prompt.\\

\textbf{Retrieval augmented generation (RAG).}
While text-to-image generative models have achieved remarkable results, one major drawback is their limited generation controllability. \emph{Retrieval augmentation}~\citep{sheynin2022knn, chen2022re, blattmann2022retrieval, yasunaga2023retrieval} is an emerging paradigm that provides more fine-grained control through the use of a retrieval dataset. In \emph{retrieval augmented generation} (RAG), the model is trained to condition on both the text prompt and samples retrieved from the retrieval dataset, making it possible to steer generation by modifying the retrieved samples. Aside from controllable generation, RAG has additional benefits: \textbf{1.} By scaling up the retrieval dataset, one can surpass the typical learning capacity limits of parametric models and obtain higher-quality samples without increase the model size \citep{blattmann2022retrieval}. \textbf{2.} By using retrieval datasets from different image domains, the model can adapt to a new target domain without fine-tuning \citep{sheynin2022knn, casanova2021instanceconditioned}. \textbf{3.} It is easy to satisfy security and privacy requirements such as information flow control~\citep{tiwari2023information, wutschitz2023rethinking}
and right-to-be-forgotten~\citep{ginart2019making, guo2019certified, bourtoule2021machine} for samples in the retrieval dataset. \looseness-1

\subsection{Differential Privacy}

Throughout we will consider randomized mechanisms operating on a dataset $D$ in some data space $\mathcal{X}$. Two datasets $D$ and $D'$ are said to be \textit{neighboring}, denoted $D \simeq D'$, if they differ in a single data point either in addition, removal or replacement of some $x \in D$.

\begin{definition}[DP; \citet{dwork2006calibrating}]
\label{def:approxdp}
Let $\epsilon \geq 0$ and $\delta \geq 0$. A randomized mechanism $M : \mathcal{X} \rightarrow \mathcal{O}$, is $(\epsilon, \delta)$-DP if for every pair of neighboring datasets $D \simeq D'$ and every subset $S \subseteq \mathcal{O}$, we have:
\begin{align*}
  \label{eq:puredp}
  \Pr [ M(D) \in S ] \leq e^{\epsilon} \Pr [ M(D') \in S ] + \delta \enspace.
\end{align*}
\end{definition}

The privacy guarantee we offer relies on Rényi DP, a generalization of DP measured in terms of the Rényi divergence.\looseness-1

\begin{definition}[Rényi divergence; \citet{renyi1961measures}]
  Let $\alpha > 1$. The Rényi divergence of order $\alpha$
  between two probability distributions $P$ and $Q$ on $\mathcal{X}$ is defined by:
\begin{equation*}\label{eq:Renyi_divergences}
  \mathbb{D}_{\alpha}(P||Q) \triangleq \frac{1}{\alpha - 1}
  \log \mathbb{E}_{o \sim Q }
  \left [ \frac{P(o)}{Q(o)} \right ]^\alpha.
\end{equation*}
\end{definition}

\begin{definition}[Rényi DP; \citet{mironov2017renyi}]
\label{def:rdp}
Let $\alpha > 1$ and $\epsilon \geq 0$. A randomized mechanism $M$ is $(\alpha, \epsilon)$-RDP if for all $D \simeq D'$, we have $\mathbb{D}_{\alpha}(M(D) \| M(D') ) \leq \epsilon $.
\end{definition}

A Rényi DP guarantee can be translated to an $(\epsilon,\delta)$-DP guarantee through conversion~\citep{balle2020hypothesis}. Rényi DP also benefits from simplified privacy composition, whereby the privacy budget for multiple invocations can be added together. Finally, Rényi DP inherits the post-processing property of DP, meaning that future computation over a private result cannot be made less private~\citep{mironov2017renyi}.

\paragraph{DP mechanisms.} In this work, we mainly consider the Sampled Gaussian Mechanism (SGM) which is an extension of the Gaussian Mechanism.

\begin{definition}[Gaussian Mechanism]
Let $f : \mathcal{X} \to \R^d$ be a query with global sensitivity $\Delta = \sup_{D \simeq D'} \norm{f(D) - f(D')}_2$ and $Z \sim \mathcal{N}(0, \sigma^2 I)$. 
The Gaussian Mechanism, defined as $M(D) = f(D) + Z$, 
satisfies $(\alpha, \epsilon)$-RDP with $\epsilon = \frac{\alpha \Delta^{2}}{2 \sigma^2}$ for every $\alpha > 1$ \citep{mironov2017renyi}.
\end{definition}

Given a large dataset, it is also possible to amplify the privacy guarantee by outputting a noisy answer on a subset of the dataset using SGM.

\begin{definition}[SGM; \citet{mironov2019r}] 
\label{def:sgm}
Let $f : \mathcal{X} \to \mathbb{R}^d$ be a function operating on subsets of $D$. The Sampled Gaussian Mechanism (SGM), with sampling rate $q \in (0, 1]$ and noise $\sigma > 0$ is defined by: 
\begin{equation*}
    \text{SG}_{q, \sigma}(D) \triangleq 
    f( \{ \vec{x}: \vec{x} \in D \text{, sampled w.p. } q \}) + Z \enspace,
\end{equation*}
where each $\vec{x}$ is sampled independently with probability $q$ without replacement and $Z \sim \mathcal{N}(0, \sigma^2 I)$.
\end{definition}

\begin{theorem}[SGM satisfies RDP; \citet{mironov2019r}]
Let $\text{SG}_{q, \sigma}(D)$ be defined as \cref{def:sgm} for function $f$. Then $\text{SG}_{q, \sigma}(D)$ satisfies $(\alpha, \epsilon)$-RDP with
\begin{equation*}
\epsilon \leq 
\mathbb{D}_\alpha((1-q) \mathcal{N}(0, \sigma^2)+q \mathcal{N}(1, \sigma^2) \| \mathcal{N}(0, \sigma^2)), 
\enspace
\end{equation*}
if $\| f(D) - f(D') \|_2 \leq 1 $ for adjacent $D, D' \in \cal{S}$. 
\end{theorem}
\paragraph{DP training/fine-tuning.} To apply differential privacy to machine learning, the quintessential algorithm is DP-SGD~\citep{abadi2016deep}---a differentially private version of the SGD algorithm. At each iteration, instead of computing the average gradient on a batch of training samples, DP-SGD applies SGM to the sum of norm-clipped per-sample gradients. The resulting noisy gradient can be used directly in optimization, and the total privacy cost can be derived through composition and post-processing.  Although DP-SGD appears straightforward at first, successful application of DP-SGD usually requires excessively large batch sizes to reduce the effect of noisy gradients~\citep{li2021large, de2022unlocking, sander2023tan, yu2023vip}.
This requirement drastically increases the computational cost of training, limiting the application of DP-SGD mostly to simple datasets and small models,
or for fine-tuning a model pre-trained on public data \citep{cattan2022fine, yu2021differentially}. In the context of DP generative modeling, DP-SGD has only been successfully applied to DP fine-tuning on simple datasets such as MNIST, CIFAR-10 and CelebA~\citep{ghalebikesabi2023differentially, lyu2023differentially}.
By contrast, our method achieves differential privacy on retrieval-augmented diffusion without DP fine-tuning.\looseness-1

\section{Differentially Private RDM}
\label{sec:dp_rdm}

\begin{figure*}[ht]
\centering
\begin{subfigure}[b]{0.3\textwidth}
  \centering
 \includegraphics[width=\linewidth]{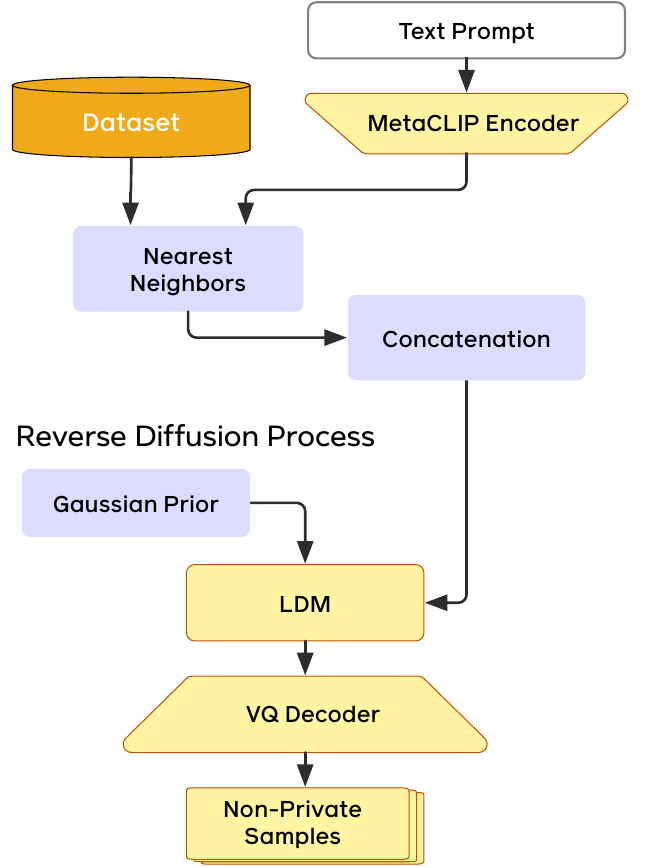}
 \caption{RDM architecture.}
 \label{fig:rdm}
\end{subfigure}
\begin{subfigure}[b]{0.6\textwidth}
 \centering
 \includegraphics[width=\linewidth]{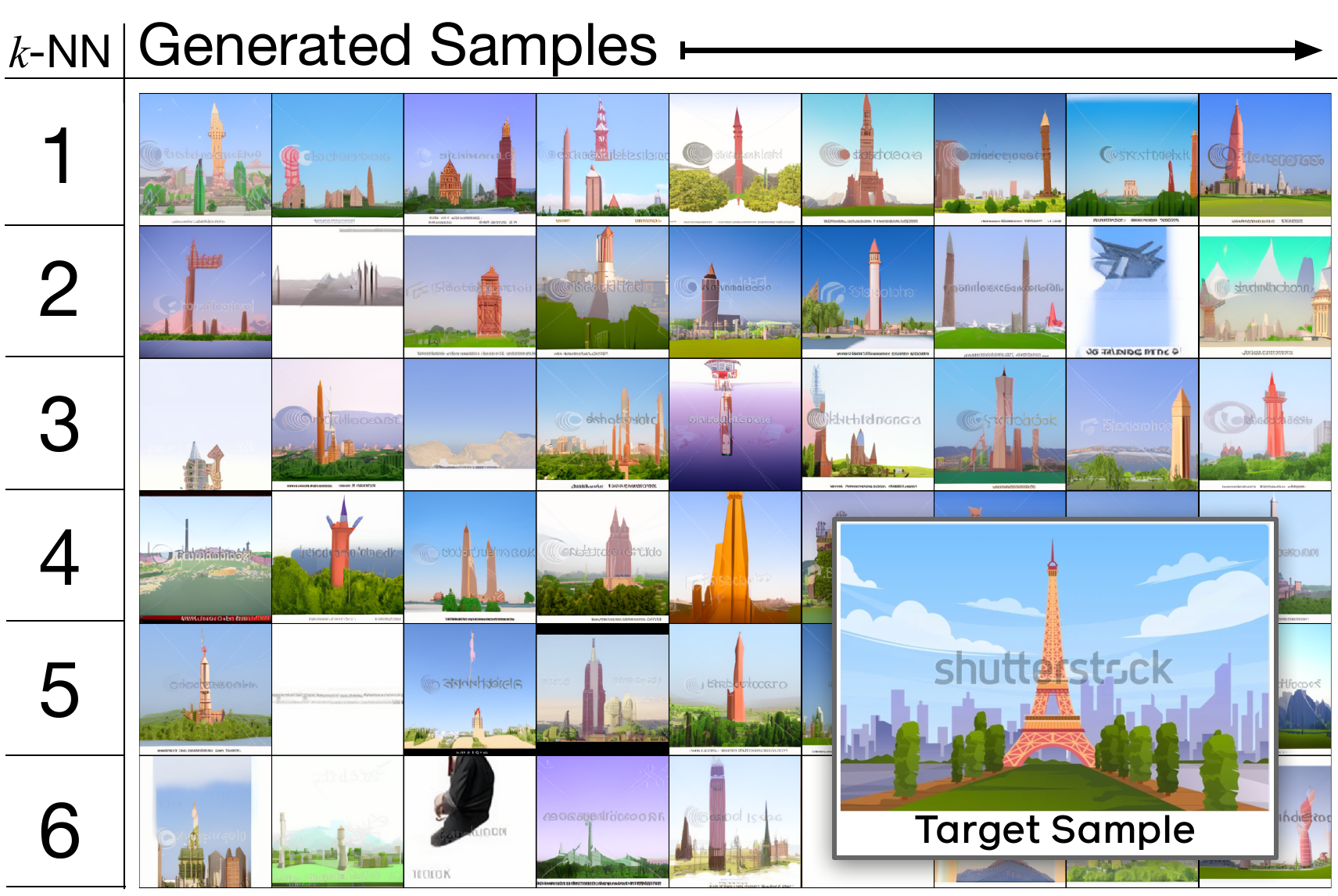}
 \caption{Samples generated by RDM using an adversarial retrieval dataset.}
 \label{fig:target_one_sample}
\end{subfigure}
\caption{(a) RDM architecture from \citet{blattmann2022retrieval}. (b) Samples generated with a non-private RDM. The retrieval dataset consists of blank images and one illustration of the Eiffel Tower with a Shutterstock watermark. Each row shows samples for a different number of retrieved neighbors $k$. 
 The watermark is clearly visible even though it came only from conditioning on the retrieval dataset.}
\end{figure*}

This section exemplifies the privacy risks in RDM, presents our \emph{differentially private retrieval-augmented diffusion model} (DP-RDM) framework, details its technical implementation and introduces privacy guarantees.

\subsection{Motivation: Privacy Risks in Retrieval Augmented Generation}
\label{sec:privacy_attack}

To motivate our DP-RDM framework, we first show that RDM~\citep{blattmann2022retrieval}, like standard text-to-image diffusion models, are also vulnerable to sample memorization.
The RDM architecture is illustrated in \cref{fig:rdm}.
To generate an image sample, RDM encodes a text prompt using the CLIP text encoder~\citep{radford2021learning} and retrieves $k$ nearest neighbors from the retrieval dataset, which contains CLIP-embedded images. The retrieved neighbors are then used as conditioning vectors in lieu of the text prompt for the diffusion model.\looseness-1

We consider information stored in the retrieval dataset to be private. To illustrate the memorization risk in RDMs, we craft an adversarial retrieval dataset that consists of an image of the Eiffel Tower with a Shutterstock watermark (see \cref{fig:target_one_sample}) and a number of blank images. We vary the number of retrieved neighbors $k$ and show image samples generated by the RDM in \cref{fig:target_one_sample}.
Note that the watermark is visible in the generated images even when $k=6$, which is a clear indication of sample-level information leakage. While this experiment was constructed with worst-case assumptions, we argue that such assumptions are reasonable when the data coverage of the retrieval dataset is sparse, and the queries can be chosen adversarially.

\subsection{DP-RDM Architecture}
\label{ssec:architecture}
Although RDM fails to provide privacy protection in the worst case, its architecture is already amenable to fine-grained control of information leakage.
Samples in the retrieval dataset leak information only through the retrieval mechanism, and influence on the generated sample can be attributed precisely to the retrieved conditioning vectors. Thus by building a ``privacy boundary'' around the retrieval mechanism and bounding the influence of each sample on the conditioning vector, we can ensure rigorous and quantifiable privacy protection through DP (see~\cref{ssec:privacy_guarantee}). 

\begin{figure}[ht]
    \centering
    \begin{minipage}{.45\linewidth}
    \centering
    \includegraphics[width=0.9\linewidth]{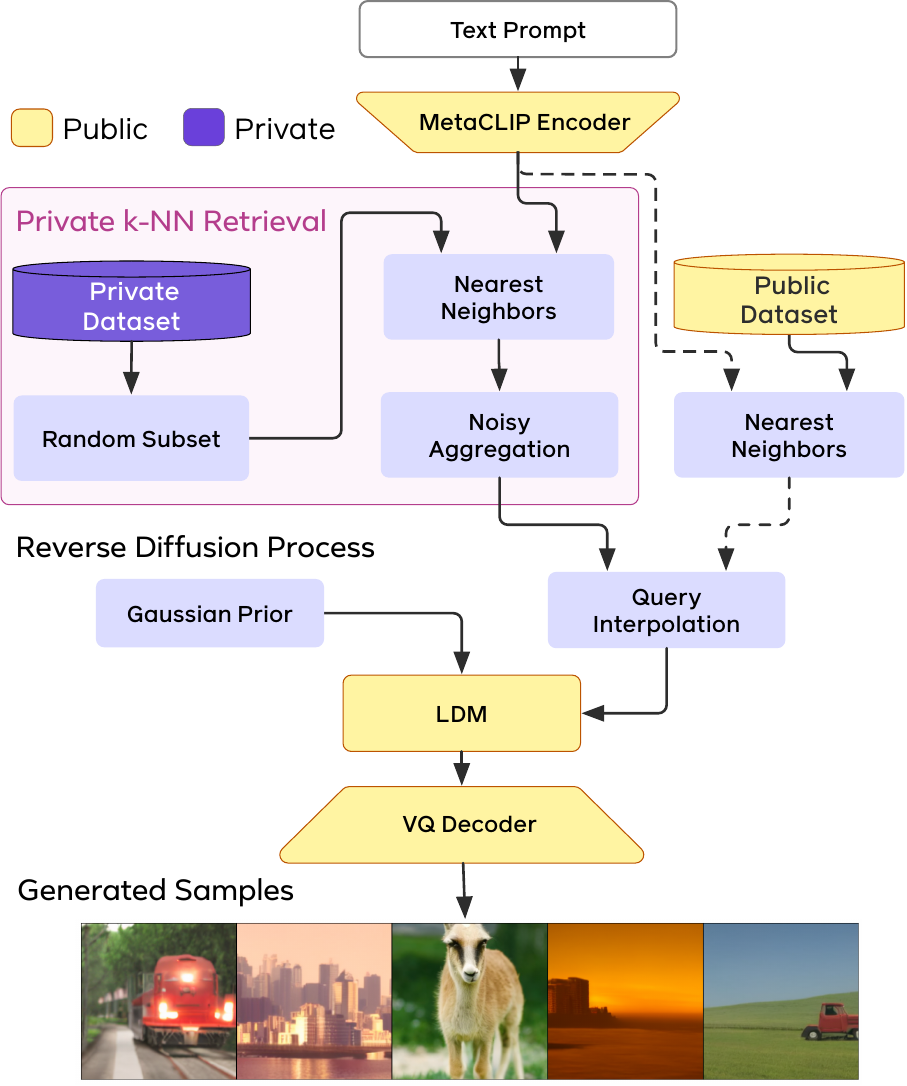}
    \caption{Text-to-image generation with DP-RDM using a private retrieval dataset. Yellow blocks refer to models trained on public datasets. The private $k$-NN block, denoted in pink, illustrates the privacy boundary between public and private data.}
    \label{fig:arch_sample_generation}
    
    \end{minipage}
    \hfill
    \begin{minipage}{.50\linewidth}
        \begin{algorithm}[H]
            \SetAlgoLined
            \SetKwInOut{Input}{input}
            \Input{LDM $\phi$,
            public retrieval dataset $U$;
            private retrieval dataset $D$; text prompt $\vec{y} \in \mathbb{R}^d$; prior $\xi \sim \mathcal{N}(0, I)$}
            \Input{Noise parameter $\sigma > 0$, num. of neighbors $k \in \mathbb{N}$, query interpolation $\lambda \in [0, 1]$, subsample rate $q \in [0, 1]$.}
            
            \textbf{$k$-NN retrieval:}
            
            $D' = \mathsf{RandomSubset}(D, q)$
        
            $N_k = \mathsf{NearestNeighbors}(D', \mathbf{y}, k)$
            
            $O_k = \mathsf{NearestNeighbors}(U, \mathbf{y}, k)$
        
            \textbf{Noisy aggregation:}
            
            $\vec{z} = \frac{1}{k} \left ( \sum_{j=1}^k N_k^{(j)} \right )+ \mathcal{N}(0, \sigma^2 I)$
        
            \textbf{Query interpolation:}
            
            $\vec{e} = (1 - \lambda) \cdot O_k + \lambda \cdot [\vec{z} | \cdots | \vec{z}] $
            
            \Return{$\phi(\vec{e}, \xi )$}
            
            \caption{Generating Samples with DP-RDM.}
            \label{alg:knn_aug_dp}
            \end{algorithm}
            \vspace{-2ex}
            \caption{Pseudo-code description of the private image generation procedure.}
    \end{minipage}
\end{figure}

\Cref{fig:arch_sample_generation} depicts the DP-RDM workflow, whose steps are summarized in \cref{alg:knn_aug_dp}. The framework is based on RDM, which retrieves image embeddings from a database given an input text-query, and uses the retrieved embeddings to condition the image generation process. We endow the retrieval augmentation system with a \emph{private $k$-NN retrieval module}, which produces a privatized image embedding. This privatized image embedding may be interpolated with embeddings of public images retrieved using the text-query. The resulting interpolated embedding is used to condition the reverse diffusion process. We detail the components of the \emph{private $k$-NN retrieval module} in the following.

\paragraph{Retrieval dataset construction.} Let $D$ be a retrieval dataset constructed by encoding a private image dataset using a pre-trained CLIP image encoder~\citep{radford2021learning} trained on a public dataset $U$. Each encoded image is a unit $d$-dimensional vector.
Since the CLIP encoder was trained on a public retrieval dataset, we assume that it is available both at training and sampling time.\looseness-1   

\paragraph{k-NN retrieval with random subset selection.} Given a text prompt, we first encode it into a unit $d$-dimensional vector $\mathbf{y}$ using the CLIP text encoder. Instead of querying $D$ directly, we obtain a subset $D'$ by subsampling each encoded image in $D$ independently with probability $q \in [0,1]$. This is done to utilize privacy amplification via subsampling~\citep{mironov2019r}. We then query $D'$ using $\mathbf{y}$ via max inner product search to obtain its $k$ nearest neighbor set $N_k \subseteq D'$.\looseness-1

\paragraph{Noisy aggregation.} The retrieved set $N_k$ must be privatized before use. Rather than conditioning on these retrieved vectors individually, we aggregate them by computing the mean as $\frac{1}{k}\sum_j{N_k^{(j)}}$.
We then apply calibrated Gaussian noise with standard deviation $\sigma>0$ to obtain the privatized embedding vector $\vec{z}$.\looseness-1

\paragraph{Query interpolation.}
Since the training dataset $U$ for the RDM is public, we can use it in conjunction with the privatized embedding vector $\vec{z}$ for retrieval augmentation. That is, we retrieve the $k$ nearest neighbors $O_k$ from $U$ and combine them with $\vec{z}$ by repeating the vector $\vec{z}$ $k$ times, resulting in $\vec{e} = (1 - \lambda) \cdot O_k + \lambda \cdot [\vec{z} | \cdots | \vec{z}]$.

Note that the contribution from $D$ through $\vec{z}$ decreases as $\lambda$ decreases, which can help improve image quality when $\sigma$, from the noisy aggregation step, is large. In the extreme cases when $\lambda=0$ or $\sigma=\infty$, no private information is leaked from $D$ and we rely entirely on $U$.\looseness-1

\subsection{Model Training}
\label{sec:model_training}
Although the RDM in \citet{blattmann2022retrieval} can be readily used with the DP-RDM generation algorithm, we can drastically improve generation quality by making a minor change in the training algorithm of the RDM. Rather than conditioning on a set of retrieved embeddings from the public dataset, we can utilize the \emph{noisy aggregation} module presented in Section~\ref{ssec:architecture} during training. Hence for each training step $t$, isotropic Gaussian noise is scaled uniform random ($\sigma_t \sim \mathcal{U}(0, \sigma)$) and added to the aggregate embedding. This change helps the model to learn to extract weak signal from the noisy conditioning vector $\mathbf{e}$ at sample generation time, thereby improving the sample quality substantially. We also substitute the CLIP encoder with the open-source MetaCLIP encoder~\citep{xu2023metaclip}. See \cref{sec:app_train_details} for further details.\looseness-1

\subsection{Privacy Guarantee}
\label{ssec:privacy_guarantee}
In the theorem below, we show that running \cref{alg:knn_aug_dp} on a single query satisfies R\'{e}nyi DP. As a result, one can leverage the composition theorem of R\'{e}nyi DP~\citep{mironov2017renyi} and conversion theorem~\citep{balle2020hypothesis} to prove that handling $T$ queries using \cref{alg:knn_aug_dp} satisfies $(\epsilon,\delta)$-DP.

\begin{theorem}
\label{thm:dp_guarantee}
Alg. \ref{alg:knn_aug_dp} satisfies $(\alpha, \epsilon)$-RDP with
\begin{equation}
    \epsilon=\mathbb{D}_\alpha(q \mathcal{N}(2/k, \sigma^2) + (1-q) \mathcal{N}(0, \sigma^2)~||~\mathcal{N}(0, \sigma^2)).
    \label{eq:epsilon_bound}
\end{equation}
\end{theorem}
\begin{proof}
We first prove the DP guarantee for replacement adjacency; the proof for leave-one-out adjacency is almost identical. Let $D_1$ and $D_2$ be neighboring datasets with $D_1 \setminus D_2 = \vec{x}_1, D_2 \setminus D_1 = \vec{x}_2$ for some $\vec{x}_1,\vec{x}_2$. For any query $\mathbf{y}$, let $D_1'$ (resp. $D_2'$) be the Poisson subsampled subset of $D$ and $N_{1,k}$ (resp. $N_{2,k})$ be the k-NN set from \cref{alg:knn_aug_dp}.

Assume that the sampling of $D_1'$ and $D_2'$ share the same randomness. Then, with probability $1-q$ we have that $\vec{x}_1 \notin D_1'$ and $\vec{x}_2 \notin D_2'$, hence $N_{1,k} = N_{2,k}$ and the distribution of $\mathbf{z}$ is identical under both $D_1$ and $D_2$. With probability $q$ we have $\vec{x}_1 \in D_1'$ and $\vec{x}_2 \in D_2'$. In this case, if $\vec{x}_1 \notin N_{1,k}$ and $\vec{x}_2 \notin N_{2,k}$ then $N_{1,k} = N_{2,k}$ and again the distribution of $\mathbf{z}$ is identical under both $D_1$ and $D_2$. Otherwise, without loss of generality assume that $\vec{x}_1 \in D_1'$. Since $\| \vec{x} \|_2 = 1$ for all $\vec{x} \in D_1,D_2$, we have that:
\begin{equation*}
  \left\| \frac{1}{k} \left ( \sum_{j=1}^k N_{1,k}^{(j)} \right ) - \frac{1}{k} \left ( \sum_{j=1}^k N_{2,k}^{(j)} \right ) \right\|_2 = \frac{1}{k} \| \vec{x}_1 - \vec{x} \|_2 \leq \frac{2}{k} \enspace ,
\end{equation*}
for some $\vec{x} \in D_2$. Thus Alg. \ref{alg:knn_aug_dp} is a subsampled Gaussian mechanism with $L_2$ sensitivity of $2/k$. Eq. \ref{eq:epsilon_bound} holds by Theorems 4 and 5 of \citet{mironov2019r}.
\end{proof}

The statement of Theorem \ref{thm:dp_guarantee} gives the RDP parameter $\epsilon$ in terms of the R\'{e}nyi divergence between a Gaussian and a mixture of Gaussians. In practice, this divergence can be computed numerically using popular packages such as Opacus~\cite{yousefpour2021opacus} and \href{https://github.com/tensorflow/privacy}{TensorFlow Privacy}.

\paragraph{Privacy-utility analysis.}
The privacy-utility trade-off of Alg.~\ref{alg:knn_aug_dp} depends on three factors: \textbf{1.} the number of neighbors $k$, \textbf{2.} the Gaussian noise magnitude $\sigma$, \textbf{3.} and the quality of retrieved embeddings $N_k$. Ideally, $k$ is large and all $k$ retrieved embeddings are relevant to the text query $\vec{y}$. This ensures that the privacy loss is low according to Theorem \ref{thm:dp_guarantee} and the generated image is faithful to the text query.

\begin{wrapfigure}{r}{0.55\textwidth}
    \centering
    \vspace{-3ex}
    \includegraphics[width=1\linewidth]{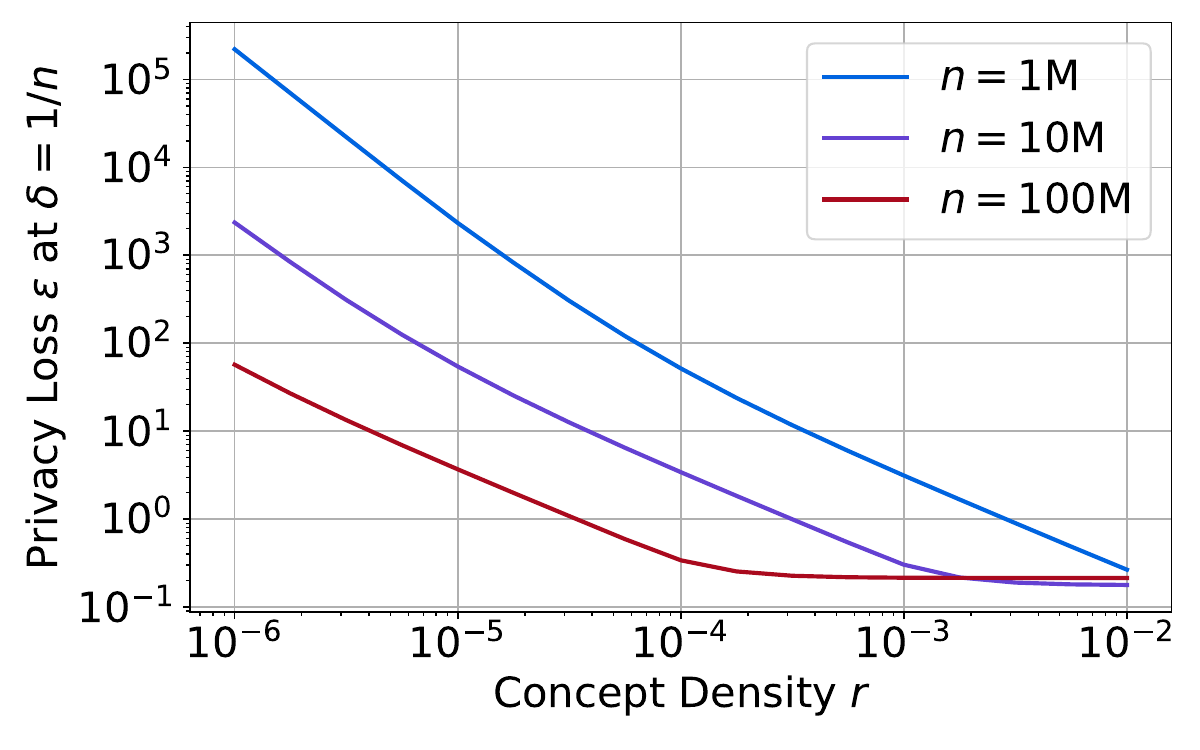}
    \vspace{-5ex}
    \caption{
    Privacy loss $\epsilon$ for generating $1,000$ images of certain concepts with concept density $r$ in log-log scale. For concepts with high concept density, it is possible to generate a large number of high-quality images under low privacy cost.
    }
    \vspace{-2ex}
    \label{fig:density_vs_epsilon_100M}
\end{wrapfigure}

To analyze this trade-off more concretely, for a given text query $\mathbf{y}$, we introduce the quantity $r \in [0,1]$ that represents the density of the concept encoded by $\mathbf{y}$ in the private retrieval dataset $D$. For example, the query ``an image of the Eiffel Tower'' likely has higher density than the query ``an image of the Eiffel Tower from 1929''. Suppose $n \triangleq |D|$ and there are at least $rn$ samples in $D$ that are relevant to the query. Under a subsampling rate of $q$, for all $k$ neighbors to be relevant, we must have $k \leq rqn$. For a fixed $\sigma$ and number of queries $T$, we can then minimize Eq. \ref{eq:epsilon_bound} over $k,q$ subject to this constraint to determine the minimum $\epsilon$ for answering $T$ queries faithfully. Intuitively, when $n$ increases, the space of $k,q$ becomes larger and it becomes possible to attain a lower $\epsilon$.

We simulate this trade-off curve for a retrieval dataset of size $n=$ 1M, 10M, 100M under a fixed $\sigma=0.1$. \cref{fig:density_vs_epsilon_100M} shows the privacy loss $\epsilon$ for generating $1,000$ images of a certain concept with density $r$. As density increases, $\epsilon$ drastically decreases as expected. At the same time, increasing the retrieval dataset size from 1M to 100M can reduce the privacy loss by as much as three orders of magnitude for rare concepts. Notably, for a common concept with density $r=10^{-3}$, $\epsilon$ can be as low as $0.21$ when $n =$ 100M. This analysis shows that the privacy-utility trade-off when generating images of common concepts using DP-RDM can be desirable, especially when scaling up the retrieval dataset.

\section{Results}
\label{sec:experiments}

To demonstrate the efficacy of our DP-RDM algorithm, we evaluate its text-to-image generation performance on several large-scale image datasets.

\subsection{Experiment Design}
\label{sec:exp_design}

\subsubsection{Datasets}

\paragraph{Evaluation datasets.} We evaluate on three image datasets: CIFAR-10~\citep{krizhevsky2009learning}, MS-COCO 2014~\citep{lin2014microsoft} with face-blurring \citep{yang2021imagenetfaces}, and Shutterstock, a privately licensed dataset of 239M image-caption pairs. Both MS-COCO and Shutterstock contain pairs of images and corresponding text descriptions, and we evaluate by sampling validation image-text pairs and using the text description to perform text-to-image generation.
For CIFAR-10, we perform class-conditional image generation using the text prompt \texttt{[label], high quality photograph} where \texttt{[label]} is replaced by one of the CIFAR-10 classes.

\paragraph{Retrieval dataset.} We consider ImageNet (with face-blurring), MS-COCO (with face-blurring), and Shutterstock as retrieval datasets. We treat ImageNet as a public dataset both for training the RDM as well as for retrieval at generation time, while MS-COCO and Shutterstock serve as private retrieval datasets. Going forward, all references to MS-COCO refer to the 2014 version with face blurring. Consequently, when the query interpolation parameter $\lambda = 0$, we leverage only public retrieval from ImageNet.

\subsubsection{Models and Baselines}
We follow the training algorithm detailed in \citet{blattmann2022retrieval} (with minor changes) to obtain two 400M parameter RDMs: \textit{RDM-fb} and \textit{RDM-adapt}. For \textit{RDM-fb}, we use the face-blurred ImageNet dataset and MetaCLIP instead of CLIP. Note that MetaCLIP is also trained with face-blurring~\citep{xu2023metaclip}. For \textit{RDM-adapt}, we further introduce noise in the retrieval mechanism during training as described in \cref{sec:model_training}, which improves the retrieval mechanism's robustness to added DP noise at generation time.

The model checkpoints above can be used in two different manners: with public retrieval (PR; \emph{c.f.} \cref{fig:rdm}) or DP retrieval with private $k$-NN (DPR; \emph{c.f.} \cref{fig:arch_sample_generation}). Applying the two retrieval methods to different retrieval datasets yields several methods/baselines:
\begin{enumerate}[nosep, leftmargin=*]
    \item \textbf{Public-only} = (\RDMFB\ or \RDMADAPT) + public retrieval dataset. This is a baseline that only utilizes the public retrieval dataset (\emph{i.e.} ImageNet), which our DP-RDM method should consistently outperform.
    \item \textbf{Non-private} = (\RDMFB\ or \RDMADAPT) + private retrieval dataset. This baseline uses public retrieval on a private dataset and is thus non-private. The performance of this baseline serves as a reference point that upper bounds the performance of DP retrieval methods.
    \item \textbf{DP-RDM} = (\DPRDMFB\ or \DPRDMADAPT) + private retrieval dataset. This is our method applied to the two RDM models.
\end{enumerate}

When using a private retrieval dataset at sample generation time, the models are evaluated using the method described in \cref{alg:knn_aug_dp}. For our experiments, we fix $\epsilon$, the query budget, $T$, and evaluate a range of neighbors $k$, subsample rate $q$ and private noise $\sigma$. For experiments with \DPRDMADAPT, we fix $\sigma$ and do binary search to find $k$ (for fixed $q$) or $q$ (for fixed $k$).
For each method, we generate 15k samples, one for each prompt from a hold-out set.

\subsubsection{Evaluation Metrics}
We measure quality, diversity and prompt-generation consistency for generated images using the Fréchet inception distance (FID) score~\citep{heusel2017gans}, CLIPScore~\citep{hessel2021clipscore}, density and coverage~\citep{naeem2020reliable}. 
The FID score measures the Fréchet distance between two image distributions by passing each image through an Inception V3 network \citep{szegedy2015going}. A lower FID score is better. The \textit{CLIPScore} quantifies how close the generated samples are to the query prompt where higher is better.  \textit{Density} and \textit{coverage} are proxies for image quality and diversity \citep{naeem2020reliable}. Both of these metrics require fixing a neighborhood of $K$ samples and measuring the nearest-neighbor distance between samples, where $\text{NND}_K $ computes the distance between sample $R_i$ and the $K^{\text{th}}$ nearest neighbor. Let $B(w, u)$ be the ball in $\mathbb{R}^d$ around point $w$ with radius $u$. 
Given real samples $R$ and fake samples $F$, then density for fixed $K$, is defined as, 
\begin{equation*}
    \text{density} \triangleq \frac{1}{K F} \sum_{j=1}^{F} \sum_{i=1}^{R} \mathbf{1}_{F_j \in B(R_i, \text{NND}_K(R_i))} \enspace .
\end{equation*}
Note that it is possible for $\text{density} > 1$. Coverage measures the existence of a generated sample near a real sample and is bounded between 0 and 1. 
Coverage is defined as,
\begin{equation*}
    \text{coverage} \triangleq \frac{1}{R} \sum_{i=1}^R \mathbf{1}_{\exists\  j \text{ s.t. } F_j \in B(R_i, \text{NND}_K(R_i))} \enspace .
\end{equation*}
We fix $K=5$ when computing coverage and density in our experiments.

\subsection{Quantitative Results}
\begin{figure*}[t]
    \centering
    \includegraphics[width=1\linewidth]{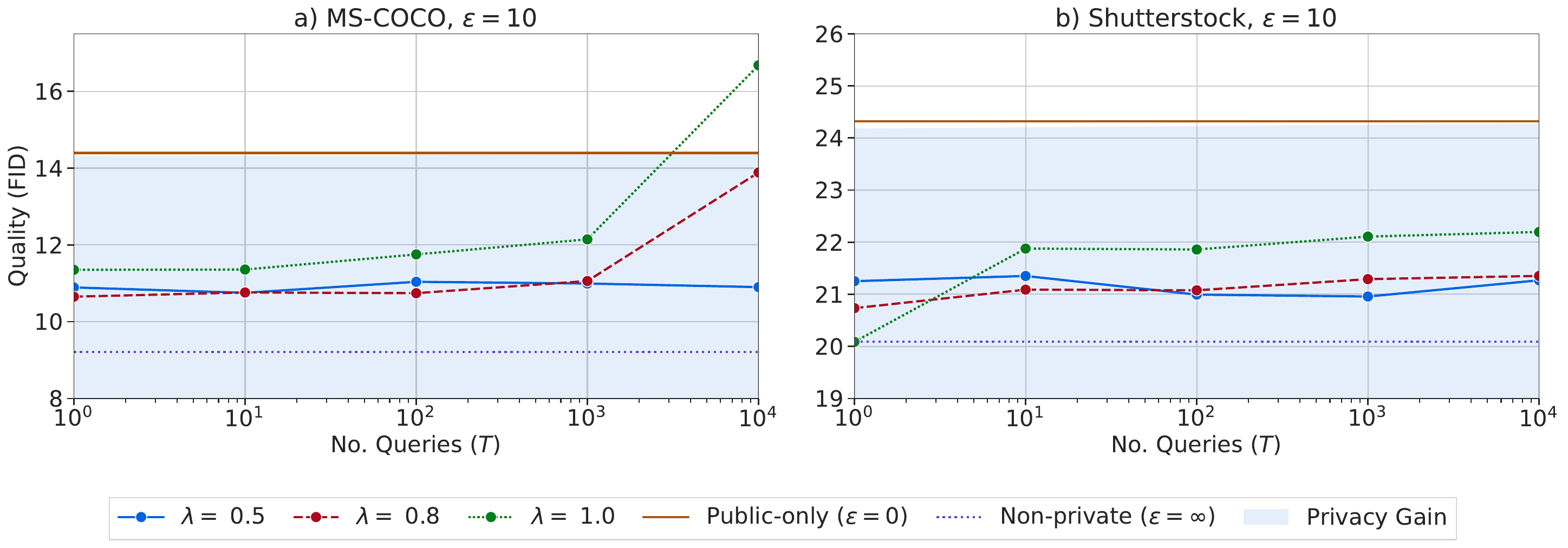}
    \caption{FID score for images generated by DP-RDM under a privacy budget of $\epsilon = 10$. %
    With both MS-COCO and Shutterstock retrieval datasets we observe a privacy gain for a wide range of number of private queries.}
    \label{fig:main_eps10}
\end{figure*}

\paragraph{FID score.} We first demonstrate the performance of DP-RDM quantitatively in terms of the Fréchet inception distance (FID) score. \cref{fig:main_eps10} shows the FID score for DP-RDM on MS-COCO and Shutterstock when answering $T$ queries privately for $T$ ranging from 1 to 10k. The evaluation dataset and the private retrieval dataset are the same, \emph{i.e.} for \cref{fig:main_eps10}(a) we treat the MS-COCO validation set as the evaluation dataset and MS-COCO training set as the private retrieval dataset. We make the following observations.
\begin{itemize}[nosep, leftmargin=*]
    \item \emph{DP-RDM achieves a desirable trade-off between model utility and privacy.} The public-only (red line) and non-private (dotted blue line) baselines define the region of interest for the FID score of DP-RDM. Any region below the non-private line is unattainable, and any region above the public-only line provides no privacy-utility benefit compared to using only the public data. For most settings of the interpolation parameter $\lambda$ and number of queries $T$, the FID score of DP-RDM sits comfortably within this region and is closer to the non-private baseline. At a high level, \cref{fig:main_eps10} suggests that at $\epsilon=10$, DP-RDM can obtain a desirable privacy-utility trade-off even for as many as 10k queries.
    \item \emph{Query interpolation helps improve model utility when answering a large number of queries.} Query interpolation with $\lambda=1$ uses only the retrieved aggregated embedding from the private retrieval dataset, which is injected with Gaussian noise to ensure differential privacy. As a result, FID suffers when answering a large number of queries. In comparison, by interpolating with the public retrieval dataset (\emph{i.e.}, $\lambda=0.5,0.8$), we see that DP-RDM retains a low FID even when answering as many as 10k queries.
    \item \emph{Increasing the size of the retrieval dataset greatly improves model utility.} The analysis in \cref{fig:density_vs_epsilon_100M} suggests that scaling up the retrieval dataset is crucial for maintaining a good privacy-utility trade-off. Our result validates this hypothesis concretely: FID score for Shutterstock stays flat when ranging $T$ from 1 to 10k, whereas FID score for MS-COCO increases sharply when $T=$ 10k for $\lambda=0.8,1$. Our DP-RDM architecture is amenable to the use of a large private retrieval dataset such as Shutterstock, whereas using such internet-scale high-fidelity datasets for DP training/finetuning of state-of-the-art generative models is out of reach at the moment.
\end{itemize}
See \cref{sec:cifar_results} for similar results with CIFAR-10, and \cref{sec_appendix:quant_results} for detailed results with $\epsilon = \{ 5, 10, 20 \}$.

\begin{figure}[t]
    \centering
    \begin{subfigure}[h]{0.32\textwidth}
         \includegraphics[width=1\linewidth]{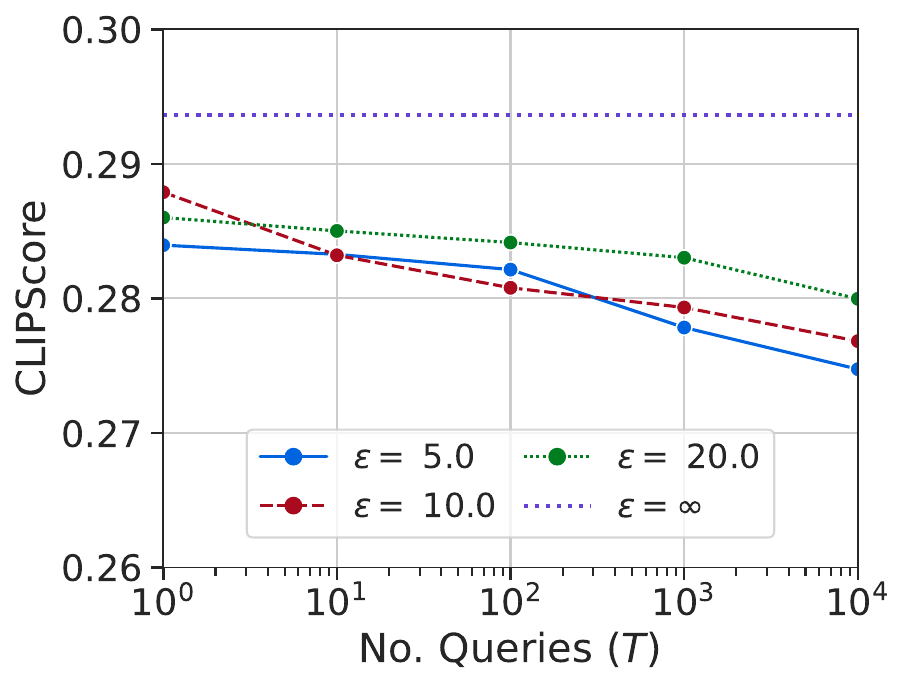}
         \label{fig:coco_clipscore}
         \caption{CLIPScore}
    \end{subfigure}
    \hfill
    \begin{subfigure}[h]{0.32\textwidth}
         \includegraphics[width=1\linewidth]{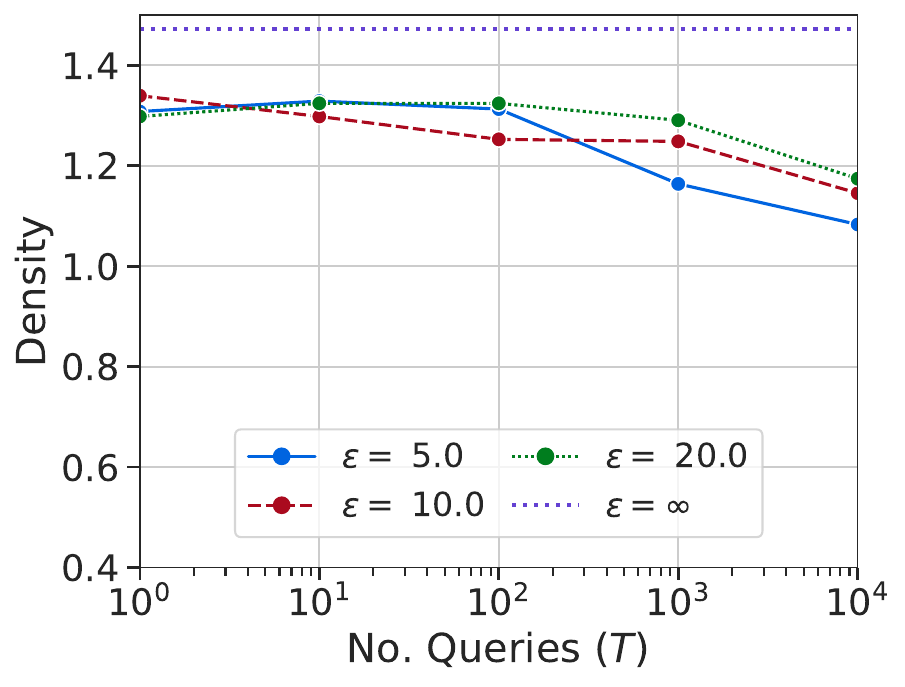}
        \caption{Density}
        \label{fig:coco_density}
    \end{subfigure}
    \hfill
    \begin{subfigure}[h]{0.32\textwidth}
        \includegraphics[width=1\linewidth]{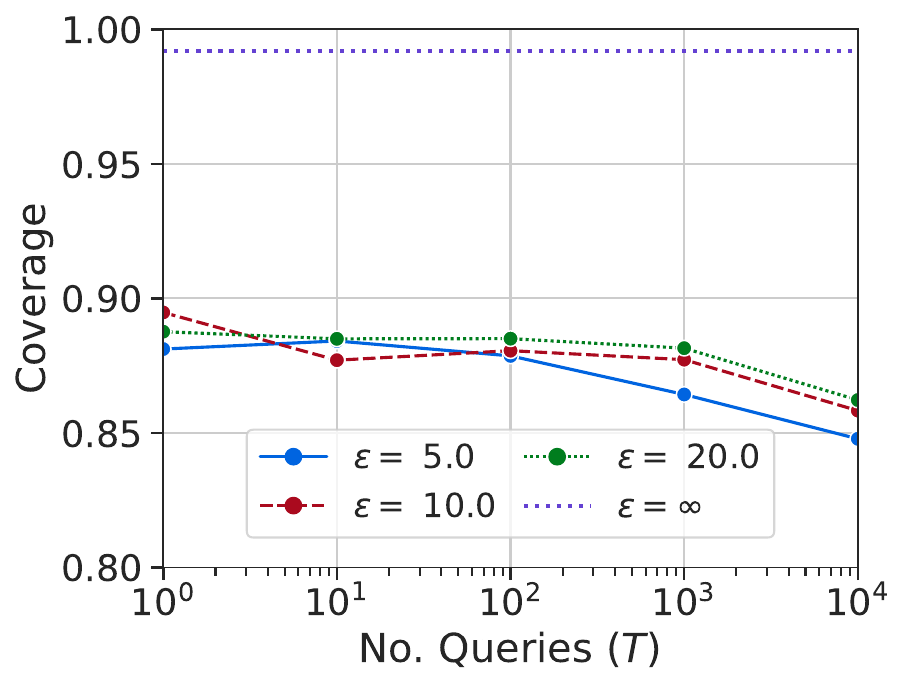}
        \caption{Coverage}
        \label{fig:coco_coverage}
    \end{subfigure}
    \caption{Quantitative evaluation of DP-RDM on MS-COCO using other metrics. CLIPScore measures the degree of semantic alignment between the text prompt and the generated images, whereas density and coverage reflect the distribution-level similarity between images generated from validation prompts and ground-truth validation images. $\epsilon=\infty$ denotes the non-private baseline.
    }
    \label{fig:coco_other_metrics}
\end{figure}

\paragraph{Other metrics.} In \cref{fig:coco_other_metrics} we plot the other quantitative metrics, namely CLIPScore, density and coverage, as a function of the number of queries for different values of $\epsilon$ on MS-COCO. See \cref{fig:st_other_metrics} in \cref{sec:more_exp_shutterstock} for corresponding plots for Shutterstock. CLIPScore reflects the degree of semantic alignment between the text prompt and the generated image, whereas density and coverage reflect the distribution-level similarity between images generated from validation prompts and ground-truth validation images.
As expected, all three metrics decrease as the number of queries increases. For the coverage metric, the gap between DP-RDM and the non-private baseline is particularly large, which suggests that DP-RDM may be suffering from mode-dropping.

\subsection{Qualitative Results}
In this section, we showcase prompts and privately generated image samples using \emph{RDM-adapt} with DP retrieval. We sample 15k validation prompts from Shutterstock and MS-COCO and use \DPRDMADAPT\ with $\epsilon=10, T=100$ to generate corresponding images, and then select the top 12 for each dataset according to CLIPScore. \cref{fig:samples-tbl-captions} shows the selected prompts and \cref{fig:samples-img-prompt} shows the generated images.
Intriguing, DP-RDM is able to generate images of concepts that are unknown to \emph{RDM-adapt} at training time by utilizing the private retrieval dataset. For example, the prompt ``abstract light background'' (prompt 9 for Shutterstock) generates a surprisingly coherent image that is unlikely to appear in the ImageNet dataset. These results suggest that DP-RDM can successfully adapt to a new private image dataset under a strong DP guarantee.
In \cref{sec:appendix_failure_cases} we also show prompts and generated images that scored the lowest in CLIPScore as failure cases.\looseness-1

\begin{smaller}
\begin{figure}[H]
    \begin{subfigure}[h]{0.45\textwidth}
            \centering
            \begin{tabular}{p{0.5cm}|p{6.5cm}}
\toprule
Ref & Validation Prompt \\
\midrule
1 & image of blue sky with white cloud for background usage. \\
2 & wood texture paint yellow \\
3 & School of Jackfish in Sipadan Malaysia  \\
4 & A bunch of garlic isolated on white background \\
5 & Halloween pumpkins on wooden table on dark color background \\
6 &  cauliflower \\
7 & Closeup of golden art sculpture pagoda and golden tiered umbrella in temple against with blue sky in north of Thailand. Faith and religion concept. \\
8 & White stork hunting in a meadow \\
9 & abstract light background    \\
10 & White grunge wood close-up background. Painted old wooden wall. \\
11 & Multicolored abstract background holidays lights in motion blur image. illustration digital. \\
12 & Torn blue jean fabric texture. \\
\bottomrule
\end{tabular}

            \caption{Shutterstock.}
      \end{subfigure}
      \hfill
    \begin{subfigure}[h]{0.45\textwidth}
      \centering
      \begin{tabular}{p{0.5cm}|p{6.5cm}}
\toprule
Ref & Validation Prompt \\
\midrule
1 & A dog watching a television that is displaying a picture of a dog. \\
2 & The Asian girls are sharing conversation under a parasol. \\
3 & An aerial view of several jumbled cars on a narrow road. \\
4 & a short passenger train on a track by a grassy hill \\
5 & A kitchen scene with focus on the microwave and oven. \\
6 & An empty wooden bench near a large body of water. \\
7 & A brown bear on two legs is near the water. \\
8 & An adult and a baby elephant crossing a road. \\
9 & A man rides a motorcycle on a track. \\
10 & A bear  that is in the grass in the wild. \\
11 & A black-and-white photo of a large bench on the sidewalk. \\
12 & A view of a sink and toilet in a bathroom in the dark. \\
\bottomrule
\end{tabular}

      \caption{MS-COCO.}
    \end{subfigure}
    \caption{Text prompts for generated samples with high CLIPScore for \DPRDMADAPT.}
    \label{fig:samples-tbl-captions}
\end{figure}
\end{smaller}
\vspace{-2em}
\begin{figure}[H]
    \begin{subfigure}[h]{0.41\textwidth}
            \centering
            \includegraphics[width=1\linewidth]{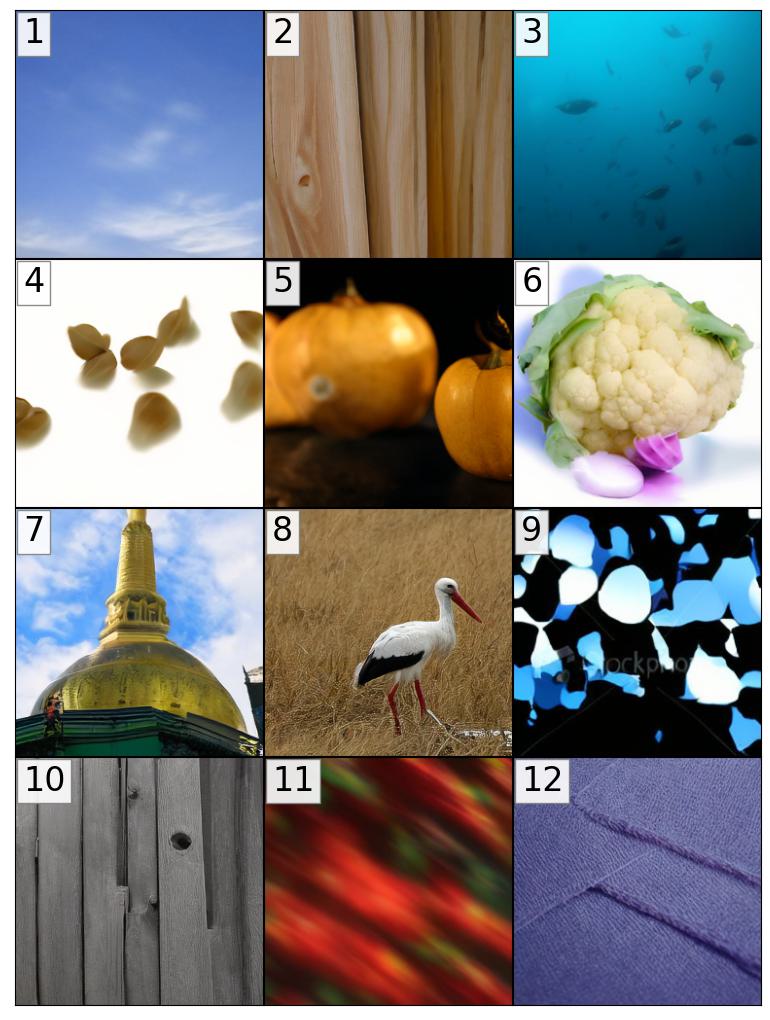}
            \caption{Shutterstock, ($k=17, \lambda=0.8, q=0.001$).}
      \end{subfigure}
      \hfill
    \begin{subfigure}[h]{0.41\textwidth}
      \centering
      \includegraphics[width=1\linewidth ]{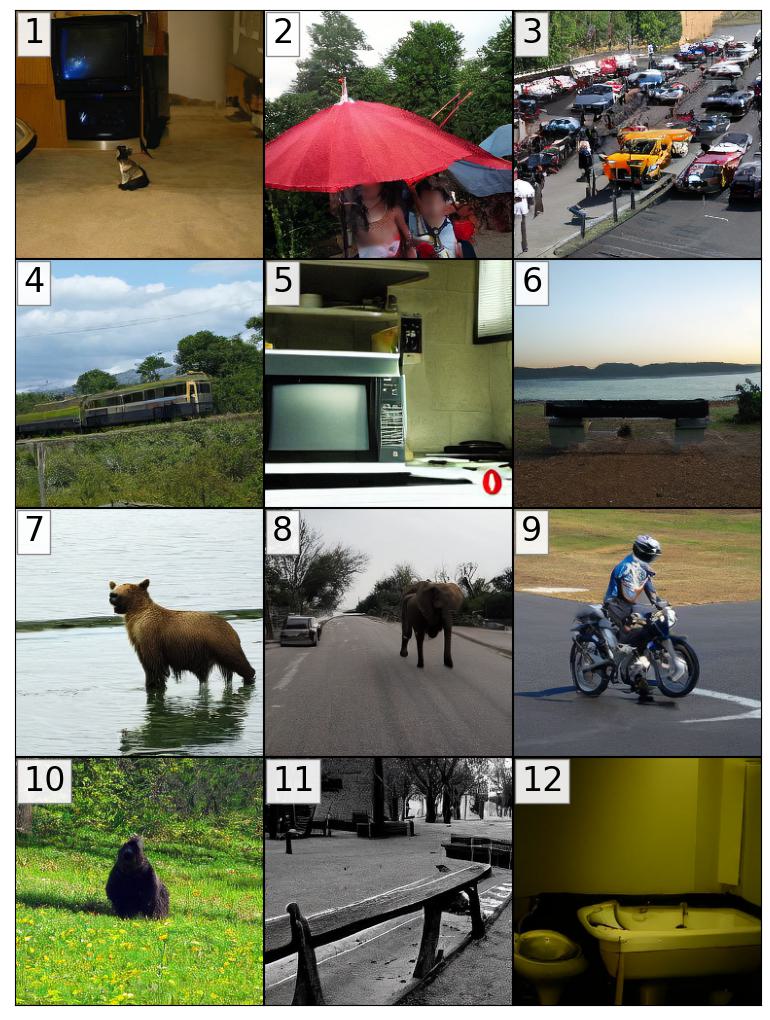}
      \caption{MS-COCO, ($k=19, \lambda=0.8, q=0.01$).}
    \end{subfigure}
    
    \caption{Images generated using \DPRDMADAPT\ with privacy budget $\epsilon=10$ under 100 queries.}
    \label{fig:samples-img-prompt}
\end{figure}

\subsection{Ablations} 

\paragraph{Effect of adapted RDM training.} In \cref{sec:model_training} we detailed changes we made to the RDM training pipeline to better adapt it to our DP-RDM algorithm. The most critical change is the addition of Gaussian noise in the aggregated embeddings during training, which helps the RDM to learn to extract weak signals from noisy embeddings. \cref{fig:plot_mscoco_fid_dprdm1} shows the FID score on MS-COCO when using the non-adapted \texttt{RDM-fb} model in DP-RDM. It is evident that without adaptation, the model is unable to generate coherent images when the number of queries $T$ increases.

\paragraph{Effect of the parameter $k$.} The number of retrieved neighbors $k$ plays a critical role in controlling the privacy-utility trade-off for DP-RDM. When $k$ is large, the privacy cost $\epsilon$ decreases at the expense of aggregating embeddings retrieved from irrelevant neighbors.
In \cref{fig:dbnorm_mscoco} and \cref{fig:dbnorm_shutterstock}, we plot the histogram of computed $L_2$ norms on a fixed sample of MS-COCO validation prompts. Intuitively, a high $L_2$ norm indicates strong semantic coherence between the retrieved neighbors, hence the generated image is more likely to be faithful to the text prompt.
Compared to MS-COCO retrieval in \cref{fig:dbnorm_mscoco}, Shutterstock is less susceptible to the negative effect of aggregating across a large number of neighbors, despite the validation prompts coming from MS-COCO validation. We suspect that this is primarily due to the size of Shutterstock being $\approx 2,900$ times larger than the MS-COCO retrieval dataset, and partially explains the strong performance of Shutterstock in \cref{fig:main_eps10}.

\begin{figure}[t]
    \begin{subfigure}{0.29\textwidth}
    \centering
    \includegraphics[width=1\linewidth]{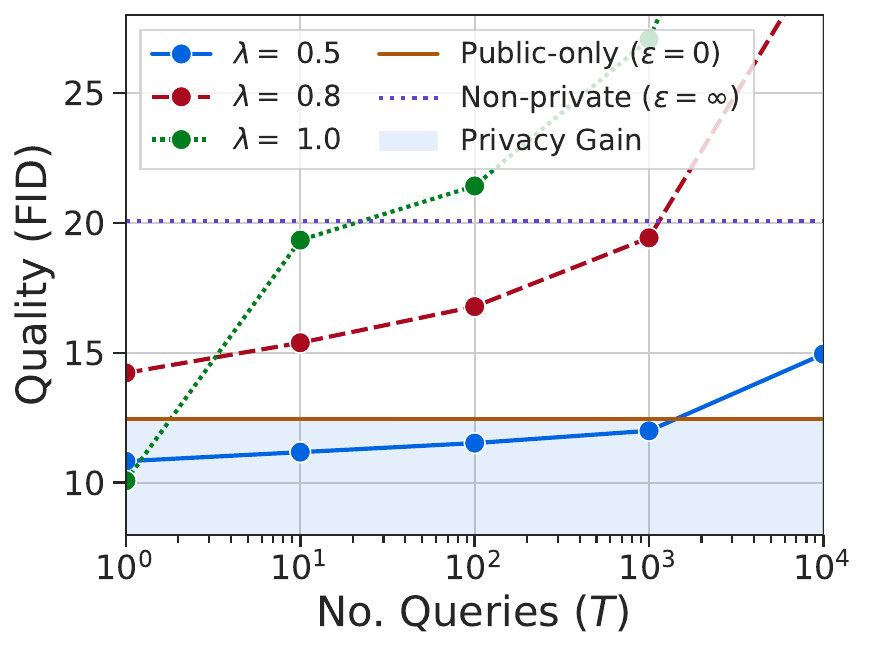}
    \caption{MS-COCO (\DPRDMFB)\looseness-1}  
    \label{fig:plot_mscoco_fid_dprdm1}
    \end{subfigure}
    \hfill
    \begin{subfigure}{0.34\textwidth}
    \centering
        \includegraphics[width=1\linewidth]{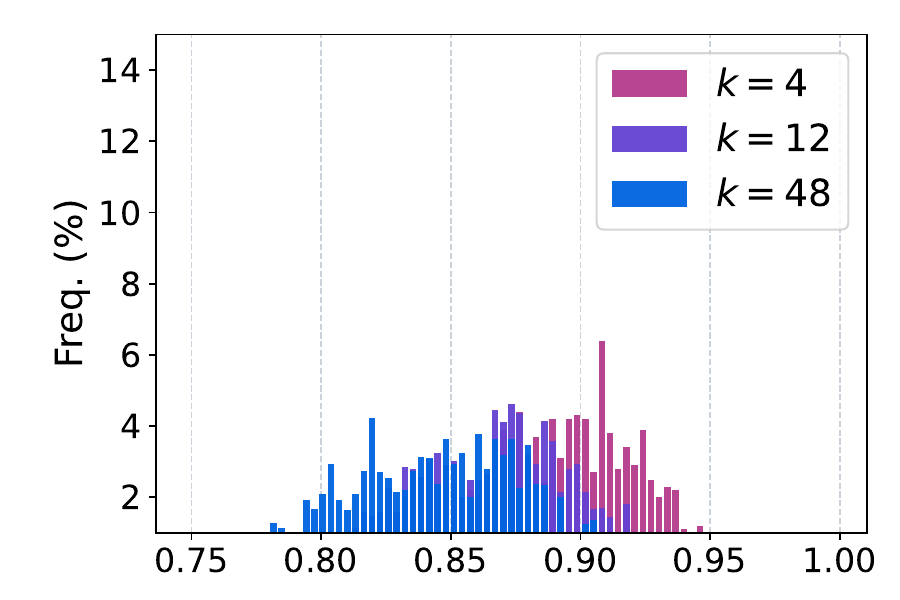}
    \caption{MS-COCO $L_2$ norm}
    \label{fig:dbnorm_mscoco}
    \end{subfigure}
    \hfill
    \begin{subfigure}{0.34\textwidth}
    \centering
        \includegraphics[width=1\linewidth]{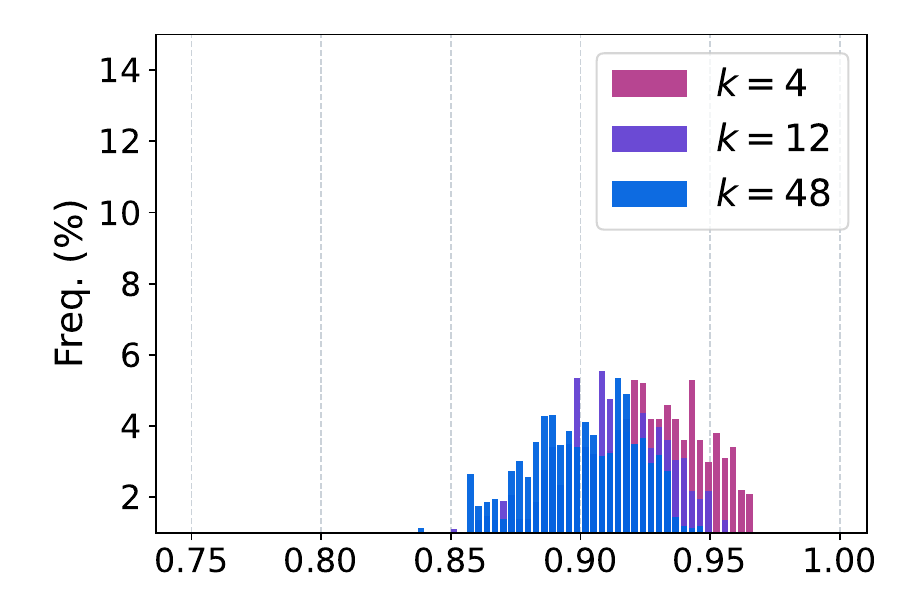}
    \caption{Shutterstock $L_2$ norm}
    \label{fig:dbnorm_shutterstock}
    \end{subfigure}
    \caption{
Ablation studies. \textbf{(a)} FID score vs. number of queries when using DP-RDM with the \texttt{RDM-fb} model. Without adapting the RDM to noise in the retrieved embeddings at training time, the model is uncompetitive compared to the public-only baseline. \textbf{(b,c)} Histogram of the $L_2$ norm of aggregated embeddings across different validation prompts. As $k$ increases, more irrelevant neighbors are included in the aggregated embedding, which reduces the retrieved embedding's $L_2$ norm.
    }
    \label{fig:retrieval_db_norm}
\end{figure}

\section{Conclusion}
We showcased DP-RDM, the first differentially private retrieval-augmented architecture for text-to-image generation. DP-RDM enables adaptation of a diffusion model trained on public data to a private domain without costly fine-tuning. By scaling up the retrieval dataset, DP-RDM can generate a large number of high-quality images (as many as 10k) under a fixed privacy budget, thereby advancing the state-of-the-art in DP image generation.

\paragraph{Limitations.} While our method achieves competitive results for DP image generation, there are some clear limitations and opportunities for future work.
\begin{enumerate}[nosep, leftmargin=*]
    \item The current privacy analysis is based on worst-case assumptions on the query and retrieval dataset. Variants of DP such as individual-level DP~\citep{feldman2021individual} offer more flexible privacy accounting on a sample-level basis, which is beneficial to DP-RDM as it can assign a different privacy budget for each sample and expend it based on the query~\citep{zhu2023private}. 
    \item RAG can easily satisfy requirements such as the right-to-be-forgotten for samples in its retrieval dataset. DP-RDM can leverage this property to guarantee both privacy and right-to-be-forgotten at the same time, although further care is needed to ensure that it interfaces well with formal data deletion definitions such as $(\epsilon,\delta)$-unlearning~\citep{ginart2019making, guo2019certified}. 
    \item Our paper focuses on private RAG for images. Can a similar technique be successfully applied to retrieval-augmented language models~\citep{lewis2020retrieval}?
\end{enumerate}

\bibliographystyle{assets/plainnat}
\bibliography{main}

\begin{thebibliography}{52}
\providecommand{\natexlab}[1]{#1}
\providecommand{\url}[1]{\texttt{#1}}
\expandafter\ifx\csname urlstyle\endcsname\relax
  \providecommand{\doi}[1]{doi: #1}\else
  \providecommand{\doi}{doi: \begingroup \urlstyle{rm}\Url}\fi

\bibitem[Abadi et~al.(2016)Abadi, Chu, Goodfellow, McMahan, Mironov, Talwar, and Zhang]{abadi2016deep}
Martin Abadi, Andy Chu, Ian Goodfellow, H~Brendan McMahan, Ilya Mironov, Kunal Talwar, and Li~Zhang.
\newblock Deep learning with differential privacy.
\newblock In \emph{Proceedings of the 2016 ACM SIGSAC conference on computer and communications security}, pages 308--318, 2016.

\bibitem[Balle et~al.(2020)Balle, Barthe, Gaboardi, Hsu, and Sato]{balle2020hypothesis}
Borja Balle, Gilles Barthe, Marco Gaboardi, Justin Hsu, and Tetsuya Sato.
\newblock Hypothesis testing interpretations and renyi differential privacy.
\newblock In \emph{International Conference on Artificial Intelligence and Statistics}, pages 2496--2506. PMLR, 2020.

\bibitem[Blattmann et~al.(2022)Blattmann, Rombach, Oktay, M{\"u}ller, and Ommer]{blattmann2022retrieval}
Andreas Blattmann, Robin Rombach, Kaan Oktay, Jonas M{\"u}ller, and Bj{\"o}rn Ommer.
\newblock Retrieval-augmented diffusion models.
\newblock \emph{Advances in Neural Information Processing Systems}, 35:\penalty0 15309--15324, 2022.

\bibitem[Bourtoule et~al.(2021)Bourtoule, Chandrasekaran, Choquette-Choo, Jia, Travers, Zhang, Lie, and Papernot]{bourtoule2021machine}
Lucas Bourtoule, Varun Chandrasekaran, Christopher~A Choquette-Choo, Hengrui Jia, Adelin Travers, Baiwu Zhang, David Lie, and Nicolas Papernot.
\newblock Machine unlearning.
\newblock In \emph{2021 IEEE Symposium on Security and Privacy (SP)}, pages 141--159. IEEE, 2021.

\bibitem[Carlini et~al.(2023)Carlini, Hayes, Nasr, Jagielski, Sehwag, Tramer, Balle, Ippolito, and Wallace]{carlini2023extracting}
Nicolas Carlini, Jamie Hayes, Milad Nasr, Matthew Jagielski, Vikash Sehwag, Florian Tramer, Borja Balle, Daphne Ippolito, and Eric Wallace.
\newblock Extracting training data from diffusion models.
\newblock In \emph{32nd USENIX Security Symposium (USENIX Security 23)}, pages 5253--5270, 2023.

\bibitem[Casanova et~al.(2021)Casanova, Careil, Verbeek, Drozdzal, and Romero]{casanova2021instanceconditioned}
Arantxa Casanova, Marlene Careil, Jakob Verbeek, Michal Drozdzal, and Adriana Romero.
\newblock Instance-conditioned {GAN}.
\newblock In A.~Beygelzimer, Y.~Dauphin, P.~Liang, and J.~Wortman Vaughan, editors, \emph{Advances in Neural Information Processing Systems}, 2021.

\bibitem[Cattan et~al.(2022)Cattan, Choquette-Choo, Papernot, and Thakurta]{cattan2022fine}
Yannis Cattan, Christopher~A Choquette-Choo, Nicolas Papernot, and Abhradeep Thakurta.
\newblock Fine-tuning with differential privacy necessitates an additional hyperparameter search.
\newblock \emph{arXiv preprint arXiv:2210.02156}, 2022.

\bibitem[Chen et~al.(2022)Chen, Hu, Saharia, and Cohen]{chen2022re}
Wenhu Chen, Hexiang Hu, Chitwan Saharia, and William~W Cohen.
\newblock Re-imagen: Retrieval-augmented text-to-image generator.
\newblock \emph{arXiv preprint arXiv:2209.14491}, 2022.

\bibitem[Chrabaszcz et~al.(2017)Chrabaszcz, Loshchilov, and Hutter]{chrabaszcz2017downsampled}
Patryk Chrabaszcz, Ilya Loshchilov, and Frank Hutter.
\newblock A downsampled variant of imagenet as an alternative to the cifar datasets.
\newblock \emph{arXiv preprint arXiv:1707.08819}, 2017.

\bibitem[De et~al.(2022)De, Berrada, Hayes, Smith, and Balle]{de2022unlocking}
Soham De, Leonard Berrada, Jamie Hayes, Samuel~L Smith, and Borja Balle.
\newblock Unlocking high-accuracy differentially private image classification through scale.
\newblock \emph{arXiv preprint arXiv:2204.13650}, 2022.

\bibitem[Dockhorn et~al.(2022)Dockhorn, Cao, Vahdat, and Kreis]{dockhorn2022differentially}
Tim Dockhorn, Tianshi Cao, Arash Vahdat, and Karsten Kreis.
\newblock Differentially private diffusion models.
\newblock \emph{arXiv preprint arXiv:2210.09929}, 2022.

\bibitem[Dwork et~al.(2006)Dwork, McSherry, Nissim, and Smith]{dwork2006calibrating}
Cynthia Dwork, Frank McSherry, Kobbi Nissim, and Adam Smith.
\newblock Calibrating noise to sensitivity in private data analysis.
\newblock In \emph{Theory of Cryptography: Third Theory of Cryptography Conference, TCC 2006, New York, NY, USA, March 4-7, 2006. Proceedings 3}, pages 265--284. Springer, 2006.

\bibitem[Feldman and Zrnic(2021)]{feldman2021individual}
Vitaly Feldman and Tijana Zrnic.
\newblock Individual privacy accounting via a renyi filter.
\newblock \emph{Advances in Neural Information Processing Systems}, 34:\penalty0 28080--28091, 2021.

\bibitem[Ghalebikesabi et~al.(2023)Ghalebikesabi, Berrada, Gowal, Ktena, Stanforth, Hayes, De, Smith, Wiles, and Balle]{ghalebikesabi2023differentially}
Sahra Ghalebikesabi, Leonard Berrada, Sven Gowal, Ira Ktena, Robert Stanforth, Jamie Hayes, Soham De, Samuel~L Smith, Olivia Wiles, and Borja Balle.
\newblock Differentially private diffusion models generate useful synthetic images.
\newblock \emph{arXiv preprint arXiv:2302.13861}, 2023.

\bibitem[Ginart et~al.(2019)Ginart, Guan, Valiant, and Zou]{ginart2019making}
Antonio Ginart, Melody Guan, Gregory Valiant, and James~Y Zou.
\newblock Making ai forget you: Data deletion in machine learning.
\newblock \emph{Advances in neural information processing systems}, 32, 2019.

\bibitem[Guo et~al.(2019)Guo, Goldstein, Hannun, and Van Der~Maaten]{guo2019certified}
Chuan Guo, Tom Goldstein, Awni Hannun, and Laurens Van Der~Maaten.
\newblock Certified data removal from machine learning models.
\newblock \emph{arXiv preprint arXiv:1911.03030}, 2019.

\bibitem[Harder et~al.(2023)Harder, Jalali, Sutherland, and Park]{harder2023pre}
Frederik Harder, Milad Jalali, Danica~J Sutherland, and Mijung Park.
\newblock Pre-trained perceptual features improve differentially private image generation.
\newblock \emph{Transactions on Machine Learning Research}, 2023.

\bibitem[Hessel et~al.(2021)Hessel, Holtzman, Forbes, Le~Bras, and Choi]{hessel2021clipscore}
Jack Hessel, Ari Holtzman, Maxwell Forbes, Ronan Le~Bras, and Yejin Choi.
\newblock Clipscore: A reference-free evaluation metric for image captioning.
\newblock In \emph{Proceedings of the 2021 Conference on Empirical Methods in Natural Language Processing}, pages 7514--7528, 2021.

\bibitem[Heusel et~al.(2017)Heusel, Ramsauer, Unterthiner, Nessler, and Hochreiter]{heusel2017gans}
Martin Heusel, Hubert Ramsauer, Thomas Unterthiner, Bernhard Nessler, and Sepp Hochreiter.
\newblock Gans trained by a two time-scale update rule converge to a local nash equilibrium.
\newblock \emph{Advances in neural information processing systems}, 30, 2017.

\bibitem[Ho et~al.(2020)Ho, Jain, and Abbeel]{ho2020denoising}
Jonathan Ho, Ajay Jain, and Pieter Abbeel.
\newblock Denoising diffusion probabilistic models.
\newblock \emph{Advances in neural information processing systems}, 33:\penalty0 6840--6851, 2020.

\bibitem[Jegou et~al.(2010)Jegou, Douze, and Schmid]{jegou2010product}
Herve Jegou, Matthijs Douze, and Cordelia Schmid.
\newblock Product quantization for nearest neighbor search.
\newblock \emph{IEEE transactions on pattern analysis and machine intelligence}, 33\penalty0 (1):\penalty0 117--128, 2010.

\bibitem[Johnson et~al.(2019)Johnson, Douze, and J{\'e}gou]{johnson2019billion}
Jeff Johnson, Matthijs Douze, and Herv{\'e} J{\'e}gou.
\newblock Billion-scale similarity search with {GPUs}.
\newblock \emph{IEEE Transactions on Big Data}, 7\penalty0 (3):\penalty0 535--547, 2019.

\bibitem[Krizhevsky and Hinton(2009)]{krizhevsky2009learning}
Alex Krizhevsky and Geoffrey Hinton.
\newblock Learning multiple layers of features from tiny images.
\newblock Technical report, University of Toronto, Toronto, Ontario, 2009.

\bibitem[Krizhevsky et~al.(2012)Krizhevsky, Sutskever, and Hinton]{krizhevsky2012imagenet}
Alex Krizhevsky, Ilya Sutskever, and Geoffrey~E Hinton.
\newblock Imagenet classification with deep convolutional neural networks.
\newblock \emph{Advances in neural information processing systems}, 25, 2012.

\bibitem[Lewis et~al.(2020)Lewis, Perez, Piktus, Petroni, Karpukhin, Goyal, K{\"u}ttler, Lewis, Yih, Rockt{\"a}schel, et~al.]{lewis2020retrieval}
Patrick Lewis, Ethan Perez, Aleksandra Piktus, Fabio Petroni, Vladimir Karpukhin, Naman Goyal, Heinrich K{\"u}ttler, Mike Lewis, Wen-tau Yih, Tim Rockt{\"a}schel, et~al.
\newblock Retrieval-augmented generation for knowledge-intensive nlp tasks.
\newblock \emph{Advances in Neural Information Processing Systems}, 33:\penalty0 9459--9474, 2020.

\bibitem[Li et~al.(2021)Li, Tramer, Liang, and Hashimoto]{li2021large}
Xuechen Li, Florian Tramer, Percy Liang, and Tatsunori Hashimoto.
\newblock Large language models can be strong differentially private learners.
\newblock \emph{arXiv preprint arXiv:2110.05679}, 2021.

\bibitem[Lin et~al.(2014)Lin, Maire, Belongie, Hays, Perona, Ramanan, Doll{\'a}r, and Zitnick]{lin2014microsoft}
Tsung-Yi Lin, Michael Maire, Serge Belongie, James Hays, Pietro Perona, Deva Ramanan, Piotr Doll{\'a}r, and C~Lawrence Zitnick.
\newblock Microsoft coco: Common objects in context.
\newblock In \emph{Computer Vision--ECCV 2014: 13th European Conference, Zurich, Switzerland, September 6-12, 2014, Proceedings, Part V 13}, pages 740--755. Springer, 2014.

\bibitem[Lyu et~al.(2023)Lyu, Vinaroz, Liu, and Park]{lyu2023differentially}
Saiyue Lyu, Margarita Vinaroz, Michael~F Liu, and Mijung Park.
\newblock Differentially private latent diffusion models.
\newblock \emph{arXiv preprint arXiv:2305.15759}, 2023.

\bibitem[Mehta et~al.(2022)Mehta, Thakurta, Kurakin, and Cutkosky]{mehta2022large}
Harsh Mehta, Abhradeep Thakurta, Alexey Kurakin, and Ashok Cutkosky.
\newblock Large scale transfer learning for differentially private image classification.
\newblock \emph{arXiv preprint arXiv:2205.02973}, 2022.

\bibitem[Mironov(2017)]{mironov2017renyi}
Ilya Mironov.
\newblock R{\'e}nyi differential privacy.
\newblock In \emph{2017 IEEE 30th computer security foundations symposium (CSF)}, pages 263--275. IEEE, 2017.

\bibitem[Mironov et~al.(2019)Mironov, Talwar, and Zhang]{mironov2019r}
Ilya Mironov, Kunal Talwar, and Li~Zhang.
\newblock Rényi differential privacy of the sampled gaussian mechanism.
\newblock \emph{arXiv preprint arXiv:1908.10530}, 2019.

\bibitem[Naeem et~al.(2020)Naeem, Oh, Uh, Choi, and Yoo]{naeem2020reliable}
Muhammad~Ferjad Naeem, Seong~Joon Oh, Youngjung Uh, Yunjey Choi, and Jaejun Yoo.
\newblock Reliable fidelity and diversity metrics for generative models.
\newblock In \emph{International Conference on Machine Learning}, pages 7176--7185. PMLR, 2020.

\bibitem[Parmar et~al.(2022)Parmar, Zhang, and Zhu]{parmar2021cleanfid}
Gaurav Parmar, Richard Zhang, and Jun-Yan Zhu.
\newblock On aliased resizing and surprising subtleties in gan evaluation.
\newblock In \emph{CVPR}, 2022.

\bibitem[Radford et~al.(2021)Radford, Kim, Hallacy, Ramesh, Goh, Agarwal, Sastry, Askell, Mishkin, Clark, et~al.]{radford2021learning}
Alec Radford, Jong~Wook Kim, Chris Hallacy, Aditya Ramesh, Gabriel Goh, Sandhini Agarwal, Girish Sastry, Amanda Askell, Pamela Mishkin, Jack Clark, et~al.
\newblock Learning transferable visual models from natural language supervision.
\newblock In \emph{International conference on machine learning}, pages 8748--8763. PMLR, 2021.

\bibitem[R{\'e}nyi(1961)]{renyi1961measures}
Alfr{\'e}d R{\'e}nyi.
\newblock On measures of entropy and information.
\newblock In \emph{Proceedings of the Fourth Berkeley Symposium on Mathematical Statistics and Probability, Volume 1: Contributions to the Theory of Statistics}, volume~4, pages 547--562. University of California Press, 1961.

\bibitem[Rombach et~al.(2022)Rombach, Blattmann, Lorenz, Esser, and Ommer]{rombach2022high}
Robin Rombach, Andreas Blattmann, Dominik Lorenz, Patrick Esser, and Bj{\"o}rn Ommer.
\newblock High-resolution image synthesis with latent diffusion models.
\newblock In \emph{Proceedings of the IEEE/CVF conference on computer vision and pattern recognition}, pages 10684--10695, 2022.

\bibitem[Saharia et~al.(2022)Saharia, Chan, Saxena, Li, Whang, Denton, Ghasemipour, Gontijo~Lopes, Karagol~Ayan, Salimans, et~al.]{saharia2022photorealistic}
Chitwan Saharia, William Chan, Saurabh Saxena, Lala Li, Jay Whang, Emily~L Denton, Kamyar Ghasemipour, Raphael Gontijo~Lopes, Burcu Karagol~Ayan, Tim Salimans, et~al.
\newblock Photorealistic text-to-image diffusion models with deep language understanding.
\newblock \emph{Advances in Neural Information Processing Systems}, 35:\penalty0 36479--36494, 2022.

\bibitem[Sander et~al.(2023)Sander, Stock, and Sablayrolles]{sander2023tan}
Tom Sander, Pierre Stock, and Alexandre Sablayrolles.
\newblock Tan without a burn: Scaling laws of dp-sgd.
\newblock In \emph{International Conference on Machine Learning}, pages 29937--29949. PMLR, 2023.

\bibitem[Sheynin et~al.(2022)Sheynin, Ashual, Polyak, Singer, Gafni, Nachmani, and Taigman]{sheynin2022knn}
Shelly Sheynin, Oron Ashual, Adam Polyak, Uriel Singer, Oran Gafni, Eliya Nachmani, and Yaniv Taigman.
\newblock Knn-diffusion: Image generation via large-scale retrieval.
\newblock \emph{arXiv preprint arXiv:2204.02849}, 2022.

\bibitem[Somepalli et~al.(2023)Somepalli, Singla, Goldblum, Geiping, and Goldstein]{somepalli2023diffusion}
Gowthami Somepalli, Vasu Singla, Micah Goldblum, Jonas Geiping, and Tom Goldstein.
\newblock Diffusion art or digital forgery? investigating data replication in diffusion models.
\newblock In \emph{Proceedings of the IEEE/CVF Conference on Computer Vision and Pattern Recognition}, pages 6048--6058, 2023.

\bibitem[Song et~al.(2020)Song, Meng, and Ermon]{song2020denoising}
Jiaming Song, Chenlin Meng, and Stefano Ermon.
\newblock Denoising diffusion implicit models.
\newblock \emph{arXiv preprint arXiv:2010.02502}, 2020.

\bibitem[Szegedy et~al.(2015)Szegedy, Liu, Jia, Sermanet, Reed, Anguelov, Erhan, Vanhoucke, and Rabinovich]{szegedy2015going}
Christian Szegedy, Wei Liu, Yangqing Jia, Pierre Sermanet, Scott Reed, Dragomir Anguelov, Dumitru Erhan, Vincent Vanhoucke, and Andrew Rabinovich.
\newblock Going deeper with convolutions.
\newblock In \emph{Proceedings of the IEEE conference on computer vision and pattern recognition}, pages 1--9, 2015.

\bibitem[Tiwari et~al.(2023)Tiwari, Gururangan, Guo, Hua, Kariyappa, Gupta, Xiong, Maeng, Lee, and Suh]{tiwari2023information}
Trishita Tiwari, Suchin Gururangan, Chuan Guo, Weizhe Hua, Sanjay Kariyappa, Udit Gupta, Wenjie Xiong, Kiwan Maeng, Hsien-Hsin~S Lee, and G~Edward Suh.
\newblock Information flow control in machine learning through modular model architecture.
\newblock \emph{arXiv preprint arXiv:2306.03235}, 2023.

\bibitem[Wutschitz et~al.(2023)Wutschitz, K{\"o}pf, Paverd, Rajmohan, Salem, Tople, Zanella-B{\'e}guelin, Xia, and R{\"u}hle]{wutschitz2023rethinking}
Lukas Wutschitz, Boris K{\"o}pf, Andrew Paverd, Saravan Rajmohan, Ahmed Salem, Shruti Tople, Santiago Zanella-B{\'e}guelin, Menglin Xia, and Victor R{\"u}hle.
\newblock Rethinking privacy in machine learning pipelines from an information flow control perspective.
\newblock \emph{arXiv preprint arXiv:2311.15792}, 2023.

\bibitem[Xu et~al.(2023)Xu, Xie, Tan, Huang, Howes, Sharma, Li, Ghosh, Zettlemoyer, and Feichtenhofer]{xu2023metaclip}
Hu~Xu, Saining Xie, Xiaoqing~Ellen Tan, Po-Yao Huang, Russell Howes, Vasu Sharma, Shang-Wen Li, Gargi Ghosh, Luke Zettlemoyer, and Christoph Feichtenhofer.
\newblock Demystifying clip data.
\newblock \emph{arXiv preprint arXiv:2309.16671}, 2023.

\bibitem[Yang et~al.(2022)Yang, Yau, Fei-Fei, Deng, and Russakovsky]{yang2021imagenetfaces}
Kaiyu Yang, Jacqueline Yau, Li~Fei-Fei, Jia Deng, and Olga Russakovsky.
\newblock A study of face obfuscation in imagenet.
\newblock In \emph{International Conference on Machine Learning (ICML)}, 2022.

\bibitem[Yasunaga et~al.(2022)Yasunaga, Aghajanyan, Shi, James, Leskovec, Liang, Lewis, Zettlemoyer, and Yih]{yasunaga2023retrieval}
Michihiro Yasunaga, Armen Aghajanyan, Weijia Shi, Rich James, Jure Leskovec, Percy Liang, Mike Lewis, Luke Zettlemoyer, and Wen-tau Yih.
\newblock Retrieval-augmented multimodal language modeling.
\newblock \emph{arXiv preprint arXiv:2211.12561}, 2022.

\bibitem[Yousefpour et~al.(2021)Yousefpour, Shilov, Sablayrolles, Testuggine, Prasad, Malek, Nguyen, Ghosh, Bharadwaj, Zhao, et~al.]{yousefpour2021opacus}
Ashkan Yousefpour, Igor Shilov, Alexandre Sablayrolles, Davide Testuggine, Karthik Prasad, Mani Malek, John Nguyen, Sayan Ghosh, Akash Bharadwaj, Jessica Zhao, et~al.
\newblock Opacus: User-friendly differential privacy library in pytorch.
\newblock \emph{arXiv preprint arXiv:2109.12298}, 2021.

\bibitem[Yu et~al.(2021)Yu, Naik, Backurs, Gopi, Inan, Kamath, Kulkarni, Lee, Manoel, Wutschitz, et~al.]{yu2021differentially}
Da~Yu, Saurabh Naik, Arturs Backurs, Sivakanth Gopi, Huseyin~A Inan, Gautam Kamath, Janardhan Kulkarni, Yin~Tat Lee, Andre Manoel, Lukas Wutschitz, et~al.
\newblock Differentially private fine-tuning of language models.
\newblock In \emph{International Conference on Learning Representations}, 2021.

\bibitem[Yu et~al.(2023)Yu, Sanjabi, Ma, Chaudhuri, and Guo]{yu2023vip}
Yaodong Yu, Maziar Sanjabi, Yi~Ma, Kamalika Chaudhuri, and Chuan Guo.
\newblock Vip: A differentially private foundation model for computer vision.
\newblock \emph{arXiv preprint arXiv:2306.08842}, 2023.

\bibitem[Zhu et~al.(2020)Zhu, Yu, Chandraker, and Wang]{zhu2020private}
Yuqing Zhu, Xiang Yu, Manmohan Chandraker, and Yu-Xiang Wang.
\newblock Private-knn: Practical differential privacy for computer vision.
\newblock In \emph{Proceedings of the IEEE/CVF Conference on Computer Vision and Pattern Recognition}, pages 11854--11862, 2020.

\bibitem[Zhu et~al.(2023)Zhu, Zhao, Guo, and Wang]{zhu2023private}
Yuqing Zhu, Xuandong Zhao, Chuan Guo, and Yu-Xiang Wang.
\newblock " private prediction strikes back!''private kernelized nearest neighbors with individual renyi filter.
\newblock \emph{arXiv preprint arXiv:2306.07381}, 2023.

\end{thebibliography}

\clearpage
\newpage
\beginappendix

\section{Appendix: Training Details}
\label{sec:app_train_details}

In \cref{fig:training_block_diagram} we detail the training procedure. 

We train on a face-blurred version of ImageNet \citep{krizhevsky2012imagenet, yang2021imagenetfaces} to enhance privacy of the pre-training dataset. While this method is not exhaustive, and in some cases will blur non-human subjects, this method effectively mitigates facial memorization by the trained model. 
Both models (\RDMFB\  and \RDMADAPT\ ) were trained using single patches, center-cropped on ImageNet FB. 

\begin{figure}[ht]
    \centering
    \includegraphics[width=0.4\linewidth]{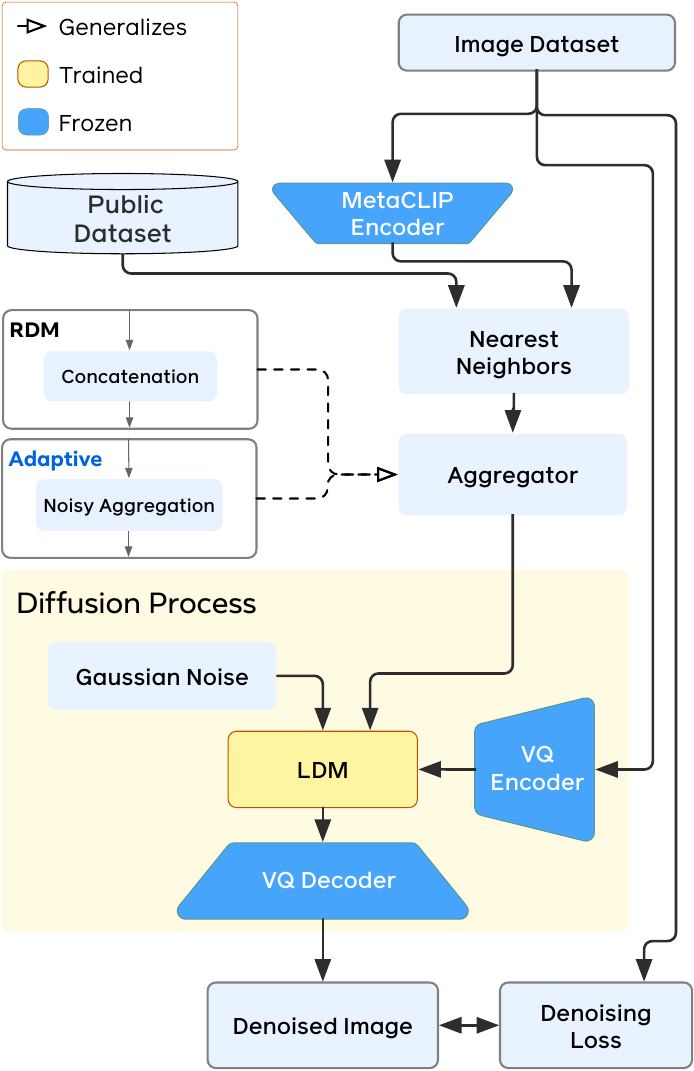}
    \vspace{-2ex}
    \caption{DP-RDM follows a typical training process for text-to-image generative models. The training process can be done entirely with public data. }
    \label{fig:training_block_diagram}
\end{figure}

DP-RDM supports two aggregation methods. \RDMFB\  instantiates the \textit{Concatenation} module, which is the standard approach for training RDM (see \citet{blattmann2022retrieval} for details). Alternatively, DP-RDM can be trained using the same \textit{Noisy Aggregation} approach done during private sample generation. Due to the bound on global sensitivity \textit{at sample generation time}, we ensure embeddings are normalized during training to match the normalization procedure when samples are generated.  

\begin{center}
\begin{figure}
    \centering
    \begin{tabular}{|r|l|l|c|}
    \hline 
        & \textbf{RDM} \citep{blattmann2022retrieval} & \textbf{\RDMFB} & \textbf{\RDMADAPT} \\
    \hline
    Steps & 120k & 110k & 140k \\
    Batch & 1280 & 3840 & 3840 \\
    \hline
    Encoder &   CLIP ViT-B/32 & \multicolumn{2}{l|}{MetaCLIP ViT-B16 (400M) }  \\
    Train Dataset & ImageNet & \multicolumn{2}{l|}{ ImageNet Face-Blurred }  \\
    \hline 
    Added Noise ($\sigma$) & \multicolumn{2}{c|}{0} & 0.05 \\
    Aggregation &      \multicolumn{2}{c|}{No} & Yes \\
    \hline
    \end{tabular}
    
    \caption{Training Setup. Experimental results were performed on models trained with face-blurring (FB) in the CLIP encoder, the conditioning and the training dataset. }
    \label{tbl:training_setup}
\end{figure}
\end{center}

\paragraph{Picking a Noise Magnitude for Training.} Remark that all embeddings are normalized, and therefore the aggregated embedding will be no greater than 1. We therefore consider the highest amount of training noise such that $\sigma \cdot \sqrt{d} \approx 1$. When training \RDMADAPT, $d = 512$ and $\sigma = 0.05$.

\paragraph{Efficient subsampling and nearest neighbor search.}
For large-scale retrieval datasets with millions of samples, it is necessary to adopt fast nearest neighbor search for computational efficiency. To this end, we use the Faiss library~\citep{johnson2019billion}, which is known to support fast inner product search for billion-scale vector databases. To support subsampling, we use the selector array functionality (\texttt{IDSelectorBatch}), which allows one to specify a binary index vector that limits the inner product search to a subset of the full retrieval dataset. We apply product quantization~\citep{jegou2010product} to the vector database in the case of Shutterstock. Note that product quantization requires training as well, which can also leak private information if done on the retrieval dataset itself. For retrieval from CIFAR-10, ImageNet FB, MS-COCO FB and the adversarial dataset, we use a flat index without index training. The Shutterstock index is built using an inverted file index and product quantization (IVFPQ) using an L2 quantizer, $8,192$  centroids, and a code size of $256$.

\paragraph{Privacy Analysis.} In \cref{fig:priv_analysis_100M} we show how DP-RDM is calibrated across a range of neighbors and noise magnitudes. This mechanism can be tuned using in terms of the subsampling rate $q$, number of neighbors $k$ and additive noise $\sigma$. \looseness-1

\begin{figure*}[ht]
    \centering
    \includegraphics[width=1\linewidth]{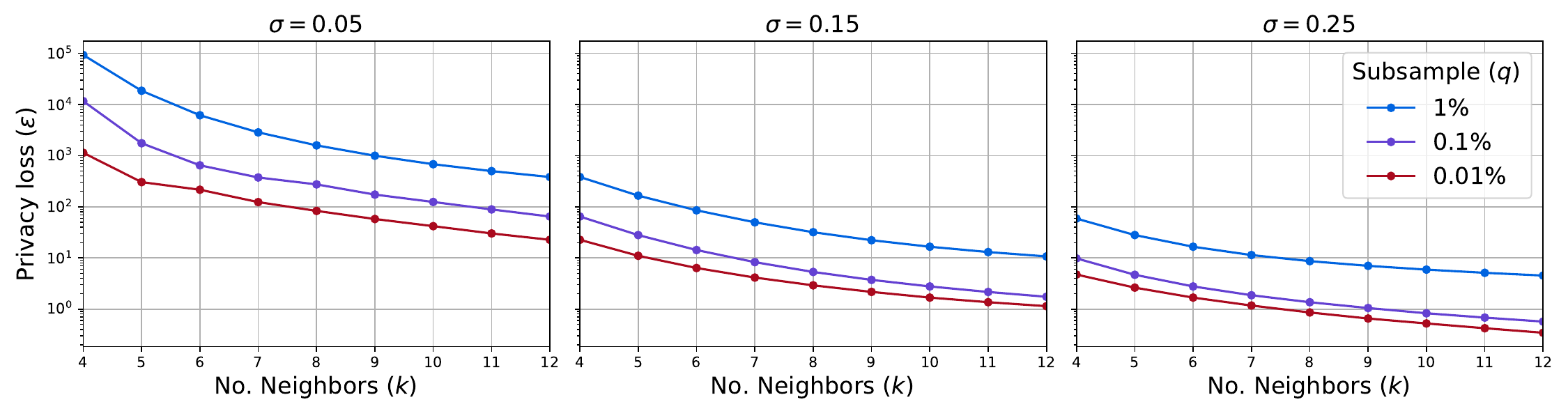}
    \caption{
    DP-RDM privacy analysis for a 100M dataset, over 10k fixed queries using privacy amplification by subsampling.
    }
    \label{fig:priv_analysis_100M}
\end{figure*}

\section{Further Experiment Details} All experiments were performed with $100$ DDIM steps. ImageNet FB, Shutterstock and MS-COCO samples were generatured using a classifier-free guidance scale of $2.0$. CIFAR-10 experiments were done with the guidance scale set to $1.25$. We resize and center crop each image before encoding the image and storing the normalized embedding. To evaluate FID, we use work by \citet{parmar2021cleanfid}. 

\subsection{Shutterstock Metrics}
\label{sec:more_exp_shutterstock}

\begin{figure}[t]
    \centering
    \begin{subfigure}[h]{0.32\textwidth}    
    \includegraphics[width=1\linewidth]{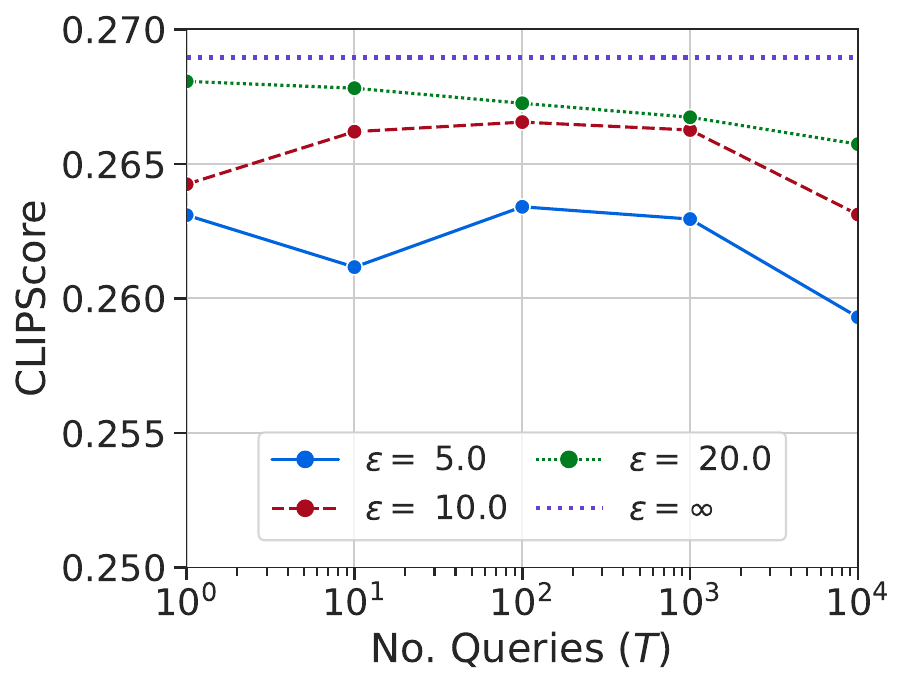}
        \caption{CLIPScore}
        \label{fig:st_clipscore}
    \end{subfigure}
    \hfill
    \begin{subfigure}[h]{0.32\textwidth}
         \includegraphics[width=1\linewidth]{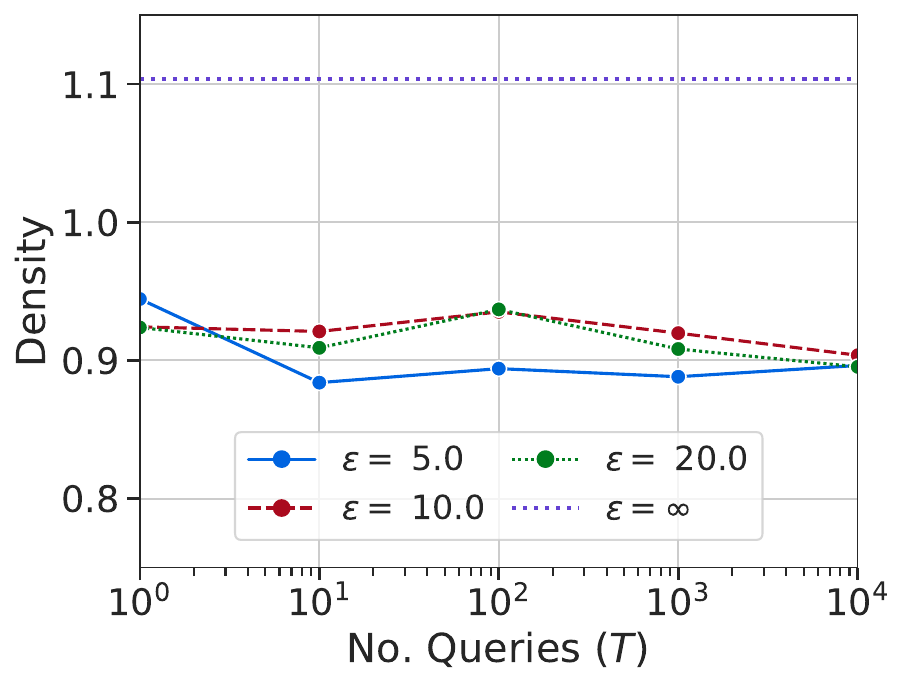}
        \caption{Density}
        \label{fig:st_density}
    \end{subfigure}
    \hfill
    \begin{subfigure}[h]{0.32\textwidth} 
        \centering   \includegraphics[width=1\linewidth]{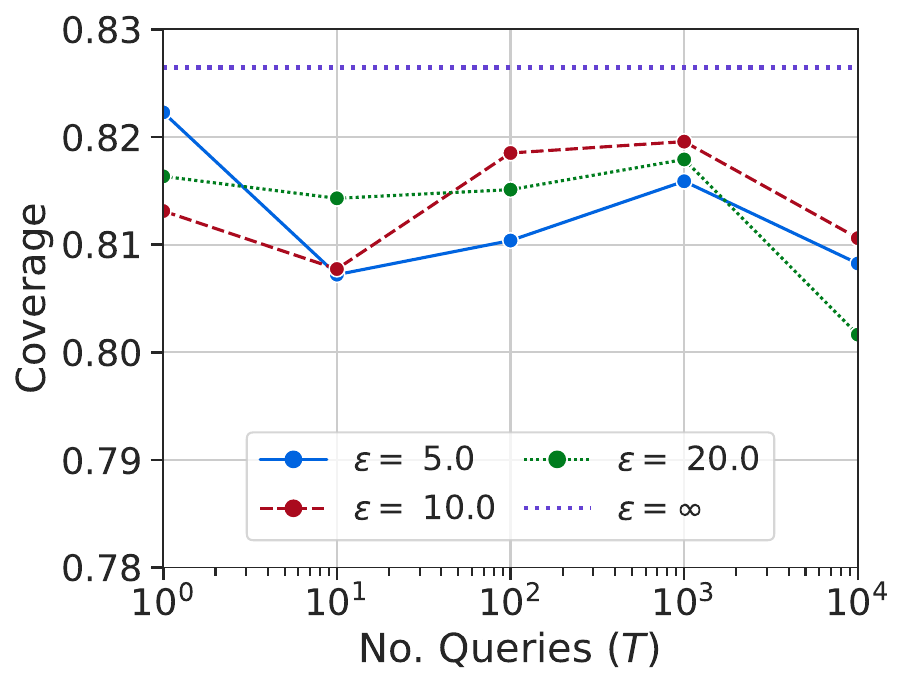}
        \caption{Coverage}
        \label{fig:st_coverage}
    \end{subfigure}
    \caption{Shutterstock Density and Coverage plots over a range of privacy parameters. The dashed line ($\epsilon = \infty$) corresponds with the non-private baseline.} 
    \label{fig:st_other_metrics}
\end{figure}

\subsection{CIFAR-10}
\label{sec:cifar_results}
CIFAR-10 images are $32 \times 32$ pixels. We therefore up-sample them to match our models. 
We found that random subsampling of the retrieval dataset (with $q = 1/100$) significantly improved quality and diversity of the images generated. We generate 30k samples and report their similarity against pre-computed statistics.  At privacy loss $\epsilon = 10$, or no privacy ($\eps = \infty$), we report the best quality score for $k$ on CIFAR-10 over 30,000 generated samples. With private diffusion models, there are two notable fine tuning results. \citet{harder2023pre} reduce the dependency on private data by using public feature extractor to train a private generator on CIFAR-10, and report FID of $26.8$.  Another fine-tuning result comes from \citet{ghalebikesabi2023differentially}, where they report FID of $9.8$ on CIFAR-10. However, both of these methods train on a downsampled version of ImageNet \citep{chrabaszcz2017downsampled}, which limits their capacity to generate high-quality samples.
 
\begin{figure}[H]
    \centering
    \begin{tabular}{llllllll}
\toprule
 &  &  & FID & $\sigma$ & $k$ & $\lambda$ & $q$ \\
Model (CIFAR-10) & $T$ & $\epsilon$ &  &  &  &  &  \\
\midrule
RDM-adapt PR & 0 & $\infty$ & 15.26 & 0.05 & 4 & 1.0 & 0.01 \\
\midrule
\multirow[t]{5}{*}{RDM-adapt DPR} & 1 & 10 & 13.91 & 0.05 & 13 & 1.0 & 0.01 \\
\cline{2-8}
 & 10 & 10 & 14.26 & 0.05 & 16 & 1.0 & 0.01 \\
\cline{2-8}
 & 100 & 10 & 14.34 & 0.05 & 19 & 1.0 & 0.01 \\
\cline{2-8}
 & 1,000 & 10 & 14.52 & 0.05 & 23 & 1.0 & 0.01 \\
\cline{2-8}
 & 10,000 & 10 & 15.19 & 0.05 & 33 & 1.0 & 0.01 \\
\cline{1-8} \cline{2-8}
\bottomrule
\end{tabular}
    
    \caption{CIFAR-10: Tabular Results}
    \label{fig:cifar_tab_results}
\end{figure}

\begin{figure}[H]
    \centering
    \includegraphics[width=0.45\linewidth]{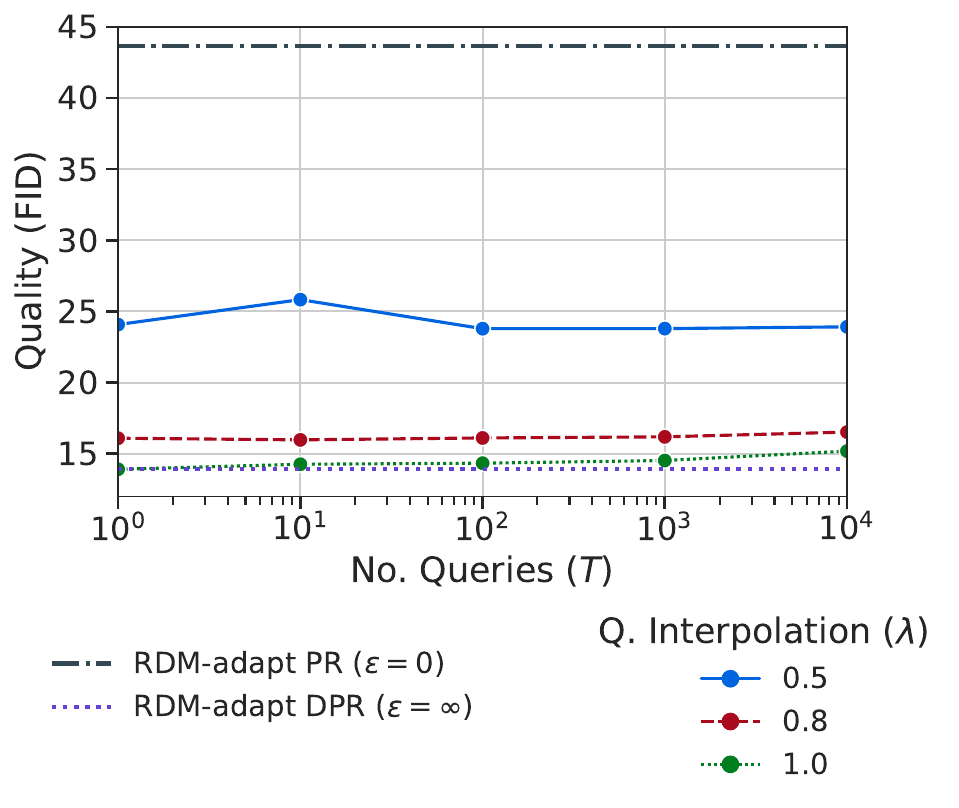}
    \caption{For $\epsilon = 10$, we compare the privacy trade-off for \DPRDMADAPT\ and CIFAR-10 retrieval.}
    \label{fig:plot_cifar10_fid}
\end{figure}

\section{Appendix: More Generated Samples}
\label{sec:appendix_samples}
\begin{figure}[H]
    \centering
    \includegraphics[width=\linewidth]{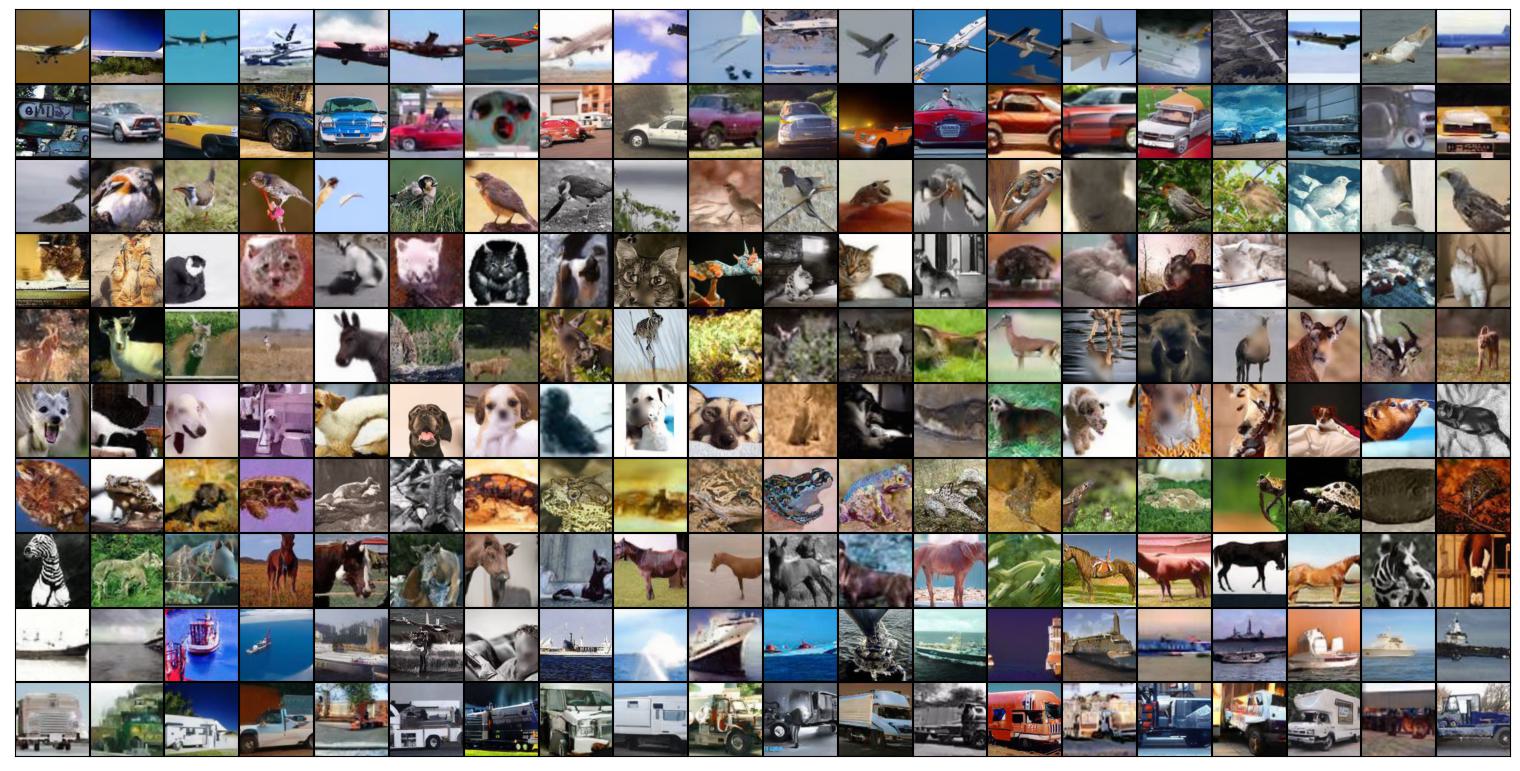}
    \caption{ \DPRDMADAPT\ 
    samples generated for $\epsilon = 10, \delta = 1/n$, where $n = 50,000$ denotes the size of the retrieval datasets (CIFAR-10 train). Here no. neighbors, $k=23$, and the Sub-sampled Gaussian Mechanism is calibrated for 1,000 private queries over random subsets $q = 0.01$.
    }
    \label{fig:samples_cifar}
\end{figure}
\newpage
\subsection{ImageNet (Face Blurred) Results}
\label{sec:imagenet_samples}
When $\lambda = 0$, then privacy loss $\epsilon = 0$ and the model relies only on the public retrieval dataset.

\begin{figure}[H]
    \begin{subfigure}[h]{0.4\textwidth}
            \centering
            \includegraphics[width=1\linewidth]{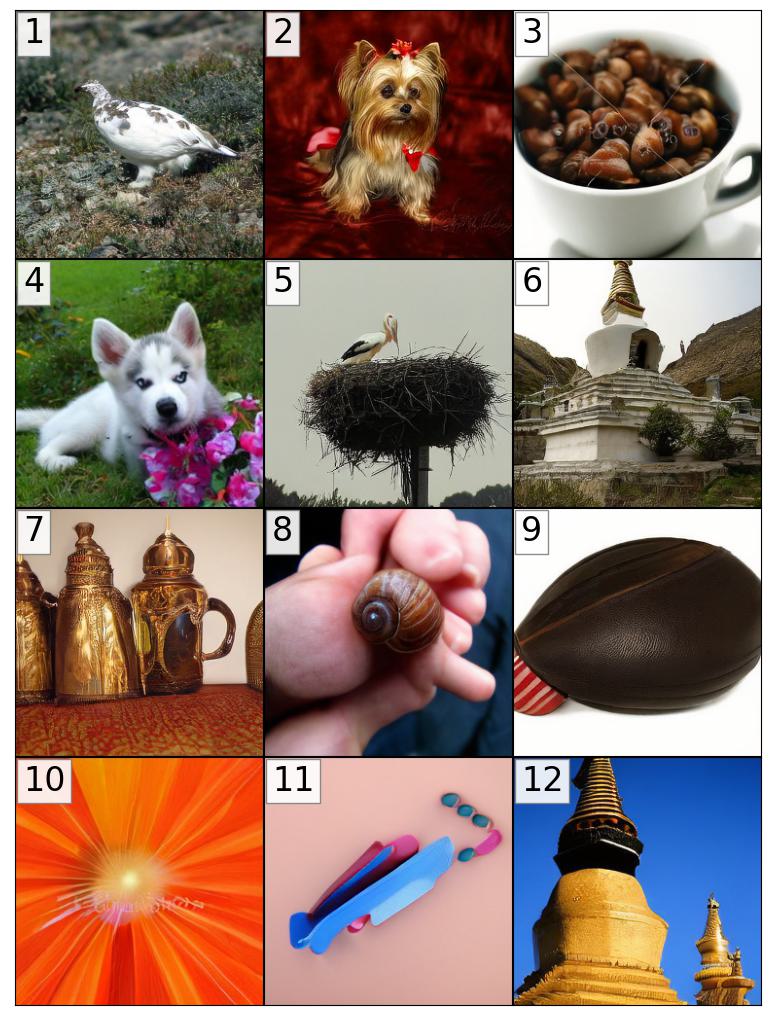}
            \caption{Shutterstock Validation}
      \end{subfigure}
      \hfill
    \begin{subfigure}[h]{0.4\textwidth}
      \centering
      \includegraphics[width=1\linewidth ]{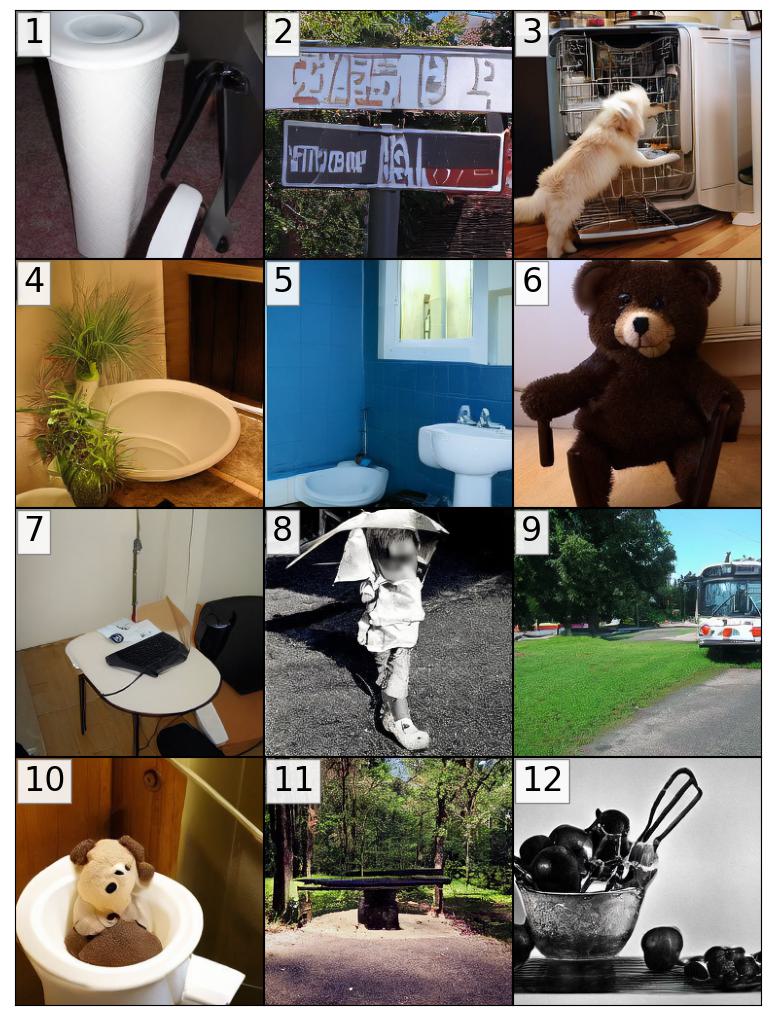}
      \caption{MS-COCO Validation}
    \end{subfigure}
    \label{fig:samples-img-prompt-inetfb}
\end{figure}
\begin{smaller}
\begin{figure}[H]
    \begin{subfigure}[h]{0.46\textwidth}
            \centering
            \begin{tabular}{p{0.25cm}|p{7.7cm}}
\toprule
Ref & Validation Prompt \\
\midrule
1 & Female Ptarmigan Stood. \\
2 & Yorkshire terrier with a red bow \\
3 & Bowl of coffee beans , isolated on white background. \\
4 & Cute Siberian Husky puppy with pink flowers \\
5 & Stork stands guarding the nestlings in his nest. Beautiful Moon in the background - Banya, Bulgaria  \\
6 & Ancient Bon stupa in Saldang village, Nepal. Saldang lies in Nankhang Valley, the most populous of the sparsely populated valleys making up the culturally Tibetan region of Dolpo. \\
7 & Turkish coffee pots made in a traditional style \\
8 & Human hand holding a big snail shell \\
9 & Rugby American football ball isolated on white background. This has clipping path. \\
10 & abstract orange. Summer background with a magnificent sun burst with lens flare. space for your message \\
11 & colored disposable knives lie on a pink and blue background \\
12 & Closeup of golden art sculpture pagoda and golden tiered umbrella in temple against with blue sky in north of Thailand. Faith and religion concept. \\
\bottomrule
\end{tabular}

            \caption{\DPRDMADAPT\ with Shutterstock.}
      \end{subfigure}
      \hfill
    \begin{subfigure}[h]{0.46\textwidth}
      \centering
      \begin{tabular}{p{0.5cm}|p{6.5cm}}
\toprule
Ref & Validation Prompt \\
\midrule
1 & A white toilet sitting next to a toilet paper roller. \\
2 & A street sign reads both in English and Asian script. \\
3 & A dog inspects an open dishwasher in a kitchen. \\
4 & A bathroom sink has a running facet and bamboo plant. \\
5 & A bathroom with blue walls and white tiles. \\
6 & A teddy bear sitting on a chair inside a room.  \\
7 & A computer desk with an open umbrella hooked up to wires. \\
8 & A young baby girl standing in a driveway holding an open umbrella. \\
9 & A bus that is sitting in the grass near a street. \\
10 & A  stuffed animal pushed into a toilet. \\
11 & A metal bench sits on a path in the forest. \\ \\
12 & Black and white photograph of a bowl of apples. \\ \\
\bottomrule
\end{tabular}

      \caption{\DPRDMADAPT\ with MS-COCO.}
    \end{subfigure}
    \label{fig:samples-tbl-prompt-appx}
\end{figure}
\end{smaller}

\subsection{Failure Cases}
\label{sec:appendix_failure_cases}
In the following figures and tables we show examples of generated samples with the lowest CLIPScore when compared to their validation prompt.
    \begin{figure}[H]
            \centering
            \includegraphics[width=1\linewidth]{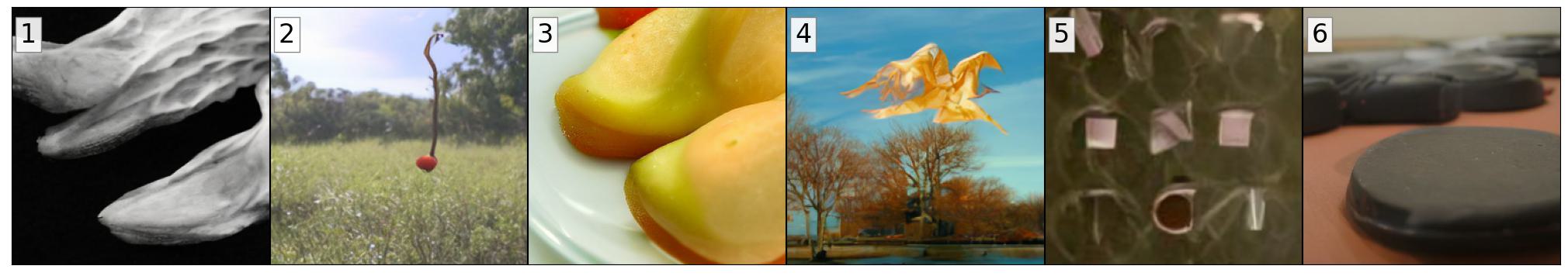}
            \caption{Shutterstock}
      \end{figure}
      
    \begin{figure}[H]
      \centering
      \includegraphics[width=1\linewidth ]{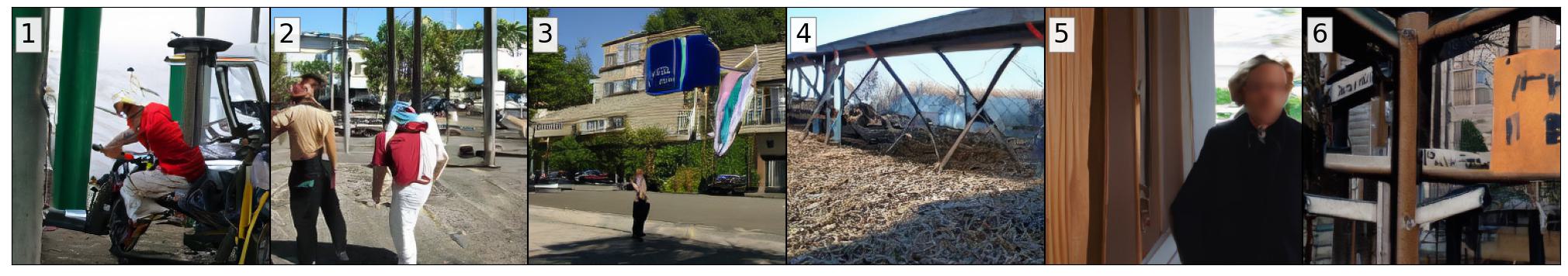}
      \caption{MS-COCO}
    \end{figure}
\begin{smaller}
\begin{figure}[H]
    \begin{subfigure}[h]{0.45\textwidth}
            \centering
            \begin{tabular}{p{0.5cm}|p{6.5cm}}
\toprule
Ref & Validation Prompt \\
\midrule
1 & Beautiful Seamless Pattern of Watercolor Summer Garden Blooming Flowers on purple background \\
2 & Hello august calligraphy and set of cute and necessary things for relaxing on the beach, isolated on white background.  \\
3 & strawberry milk splash isolated on white background \\
4 & Merry Christmas line flat design card with holidays symbols - Santa Claus, Christmas tree, house, mistletoe \\
5 & Illustration of card with bowknot with round frame for text isolated on white background. Black friday sale banner template. Design for Christmas sale, shopping, retail, discount  poster \\
6 & Men's classic shoes isolated on white background \\
\bottomrule
\end{tabular}

            \caption{Failure cases with Shutterstock.}
      \end{subfigure}
      \hfill
    \begin{subfigure}[h]{0.45\textwidth}
      \centering
      \begin{tabular}{p{0.5cm}|p{6.5cm}}
\toprule
Ref & Validation Prompt \\
\midrule
1 & An old plane flying through the cloudy sky. \\ \\ \\
2 & A person sitting by the water with a bike. \\ \\ \\ 
3 & A man in a neon green vest is holding a stop sign on the newly paved road. \\ \\
4 & Tire marks carve paths down a snow covered street. \\ 
5 & An elderly man wearing a white button down shirt and a black bow tie smiling. \\ \\ \\ \\ 
6 & A number of red and green traffic lights on a wide highway. \\
\bottomrule
\end{tabular}

      \caption{Failure cases with MS-COCO.}
    \end{subfigure}
    \label{fig:samples-tbl-prompt-failure}
\end{figure}
\end{smaller}

\newpage
\section{Appendix: RDM and DP-RDM Compared}
\label{sec:app_private_retrieval}
Here we show how \DPRDMADAPT\ mitigates individual samples from appearing in an adversarially chosen dataset via noisy aggregation. 
\begin{figure}[H]
    \centering
    \begin{subfigure}[b]{0.4\textwidth}
         \centering
         \includegraphics[width=\textwidth]{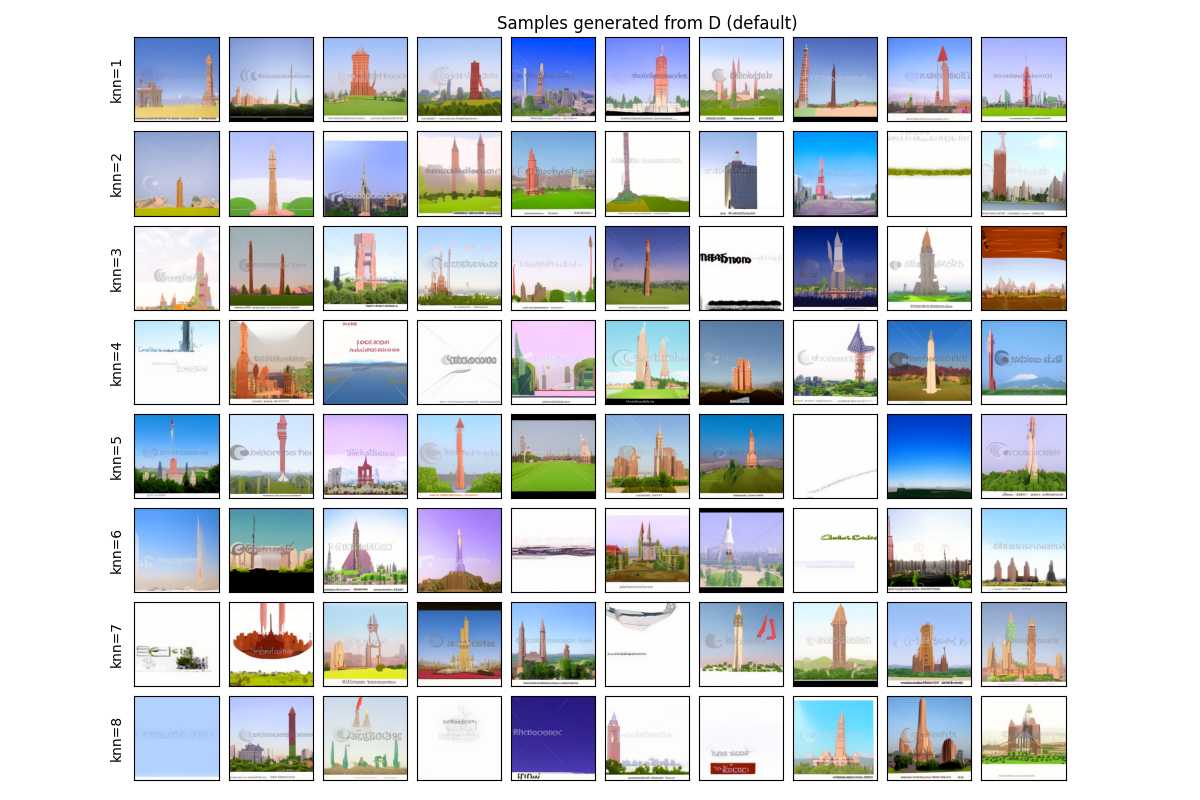}
         \caption{Adversarial $D$ with \textit{RDM}}
         \label{fig:mia_default_D}
     \end{subfigure}
     \hfill
     \begin{subfigure}[b]{0.4\textwidth}
         \centering
         \includegraphics[width=\textwidth]{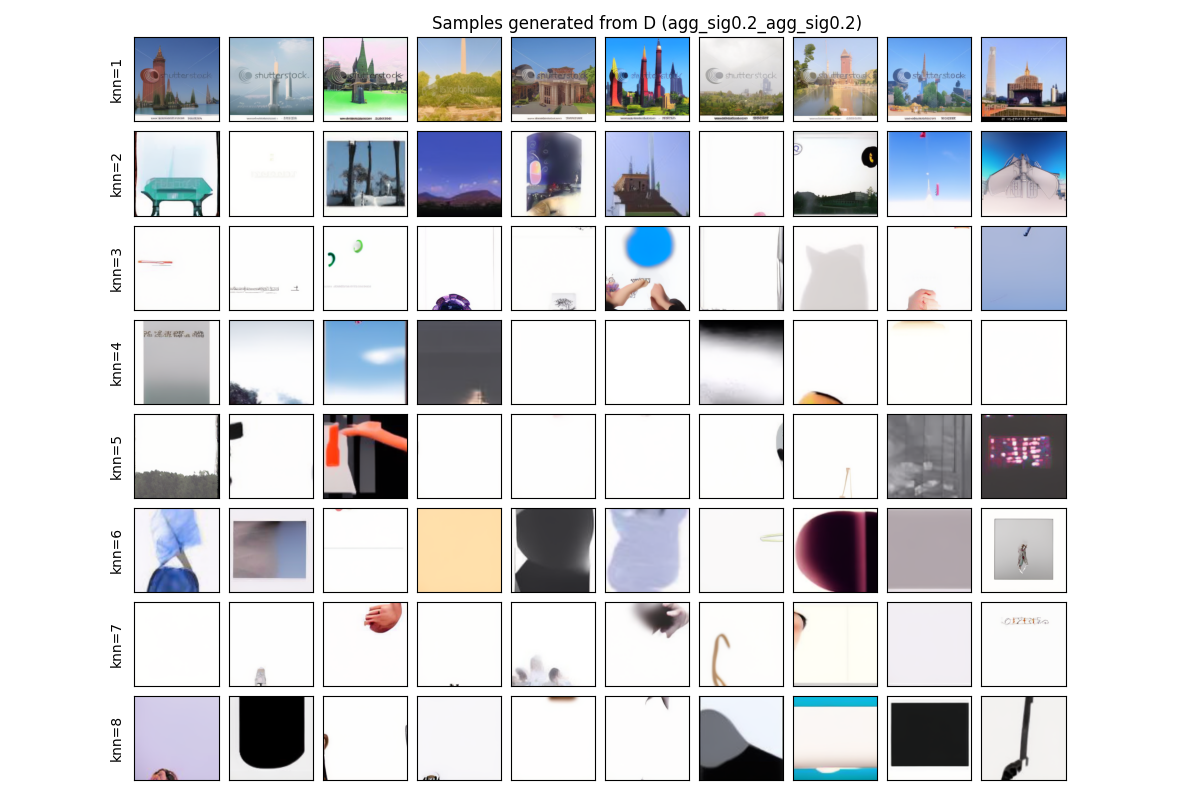}
         \caption{Adversarial $D$ with \DPRDMADAPT}
         \label{fig:mia_default_Dprime}
     \end{subfigure}
\end{figure}

\section{Appendix: Quantitative Results}
\label{sec_appendix:quant_results}
\begin{figure}[H]
    \centering
    \begin{tabular}{llrrrr}
\toprule
 &  & FID & CLIPScore & Coverage & Density \\
Model & Validation &  &  &  &  \\
\midrule
\multirow[t]{3}{*}{RDM-adapt PR} & CIFAR-10 & 40.87 & - & - & - \\
 & MS-COCO FB & 14.40 & 0.294 & 0.85 & 1.15 \\
 & Shutterstock & 24.32 & 0.269 & 0.79 & 0.95 \\
\cline{1-6}
\bottomrule
\end{tabular}

    \caption{Imagenet Face Blurred: Tabular Results}
    \label{fig:imagenet_tab_results}
\end{figure}

\begin{figure}[H]
    \centering
    \begin{tabular}{lllllllllll}
\toprule
 &  &  & FID & CLIPScore & Coverage & Density & $\sigma$ & $k$ & $\lambda$ & $q$ \\
MS-COCO FB & $T$ & $\epsilon$ &  &  &  &  &  &  &  &  \\
\midrule
RDM-adapt PR & 0 & $\infty$ & 10.08 & 0.267 & 0.88 & 1.4 & 0.0 & 4 & 1.0 & 1.0 \\
\midrule
\multirow[t]{15}{*}{RDM-adapt DPR} & \multirow[t]{3}{*}{1} & 5 & 11.14 & 0.26 & 0.86 & 1.13 & 0.05 & 23 & 1.0 & 0.05 \\
 &  & 10 & 11.35 & 0.273 & 0.88 & 1.25 & 0.04 & 21 & 1.0 & 0.1277 \\
 &  & 20 & 10.68 & 0.265 & 0.88 & 1.14 & 0.05 & 10 & 1.0 & 0.05 \\
\cline{2-11}
 & \multirow[t]{3}{*}{10} & 5 & 11.32 & 0.259 & 0.86 & 1.13 & 0.05 & 28 & 1.0 & 0.05 \\
 &  & 10 & 11.36 & 0.251 & 0.85 & 0.97 & 0.065 & 16 & 1.0 & 0.05 \\
 &  & 20 & 10.96 & 0.263 & 0.88 & 1.13 & 0.05 & 15 & 1.0 & 0.05 \\
\cline{2-11}
 & \multirow[t]{3}{*}{100} & 5 & 11.57 & 0.257 & 0.85 & 1.1 & 0.05 & 37 & 1.0 & 0.05 \\
 &  & 10 & 11.75 & 0.249 & 0.83 & 0.99 & 0.065 & 21 & 1.0 & 0.05 \\
 &  & 20 & 11.1 & 0.261 & 0.86 & 1.12 & 0.05 & 21 & 1.0 & 0.05 \\
\cline{2-11}
 & \multirow[t]{3}{*}{1,000} & 5 & 15.5 & 0.24 & 0.75 & 0.88 & 0.05 & 29 & 1.0 & 0.01 \\
 &  & 10 & 12.14 & 0.246 & 0.82 & 0.94 & 0.065 & 34 & 1.0 & 0.05 \\
 &  & 20 & 11.26 & 0.258 & 0.85 & 1.11 & 0.05 & 32 & 1.0 & 0.05 \\
\cline{2-11}
 & \multirow[t]{3}{*}{10,000} & 5 & 20.01 & 0.218 & 0.66 & 0.66 & 0.08 & 30 & 1.0 & 0.01 \\
 &  & 10 & 16.68 & 0.237 & 0.73 & 0.83 & 0.05 & 34 & 1.0 & 0.01 \\
 &  & 20 & 14.61 & 0.242 & 0.77 & 0.91 & 0.05 & 26 & 1.0 & 0.01 \\
\cline{1-11} \cline{2-11}
\multirow[t]{6}{*}{RDM-fb} & 1 & 10 & 17.9 & 0.232 & 0.94 & 0.78 & 0.03 & 22 & 1.0 & 0.01 \\
\cline{2-11}
 & 10 & 10 & 20.86 & 0.227 & 0.88 & 0.62 & 0.035 & 22 & 1.0 & 0.01 \\
\cline{2-11}
 & 100 & 10 & 21.42 & 0.227 & 0.6 & 0.61 & 0.033 & 28 & 1.0 & 0.01 \\
\cline{2-11}
 & 1,000 & 10 & 27.07 & 0.22 & 0.49 & 0.4 & 0.041 & 28 & 1.0 & 0.01 \\
\cline{2-11}
 & 10,000 & 10 & 42.01 & 0.207 & 0.3 & 0.18 & 0.059 & 28 & 1.0 & 0.01 \\
\cline{2-11}
 & 100,000 & 10 & 155.67 & 0.192 & 0.06 & 0.02 & 0.164 & 22 & 1.0 & 0.01 \\
\cline{1-11} \cline{2-11}
\bottomrule
\end{tabular}

    \caption{MS-COCO Face Blurred: Tabular Results}
    \label{fig:mscoco_tab_results}
\end{figure}

\begin{figure}[H]
    \centering
    \begin{tabular}{lllllllllll}
\toprule
 &  &  & FID & CLIPScore & Coverage & Density & $\sigma$ & $k$ & $\lambda$ & $q$ \\
Shutterstock & $T$ & $\epsilon$ &  &  &  &  &  &  &  &  \\
\midrule
RDM-adapt PR & 0 & $\infty$ & 20.09 & 0.245 & 0.77 & 0.71 & 0.05 & 4 & 1.0 & 1.0 \\
\cline{1-11} \cline{2-11}
\multirow[t]{15}{*}{RDM-adapt DPR} & \multirow[t]{3}{*}{1} & 5 & 20.77 & 0.256 & 0.8 & 0.84 & 0.08 & 16 & 0.5 & 0.01 \\
 &  & 10 & 20.73 & 0.244 & 0.77 & 0.73 & 0.065 & 14 & 0.8 & 0.01 \\
 &  & 20 & 20.95 & 0.262 & 0.8 & 0.87 & 0.06 & 10 & 0.5 & 0.01 \\
\cline{2-11}
 & \multirow[t]{3}{*}{10} & 5 & 21.19 & 0.255 & 0.8 & 0.83 & 0.08 & 17 & 0.5 & 0.01 \\
 &  & 10 & 20.93 & 0.26 & 0.8 & 0.85 & 0.065 & 16 & 0.5 & 0.01 \\
 &  & 20 & 21.05 & 0.262 & 0.8 & 0.88 & 0.06 & 12 & 0.5 & 0.01 \\
\cline{2-11}
 & \multirow[t]{3}{*}{100} & 5 & 21.28 & 0.247 & 0.78 & 0.79 & 0.06 & 24 & 0.8 & 0.01 \\
 &  & 10 & 20.99 & 0.256 & 0.8 & 0.84 & 0.08 & 14 & 0.5 & 0.01 \\
 &  & 20 & 20.86 & 0.244 & 0.78 & 0.74 & 0.065 & 13 & 0.8 & 0.01 \\
\cline{2-11}
 & \multirow[t]{3}{*}{1,000} & 5 & 21.02 & 0.247 & 0.78 & 0.75 & 0.06 & 28 & 0.8 & 0.01 \\
 &  & 10 & 20.96 & 0.26 & 0.81 & 0.88 & 0.065 & 21 & 0.5 & 0.01 \\
 &  & 20 & 21.09 & 0.244 & 0.77 & 0.74 & 0.065 & 16 & 0.8 & 0.01 \\
\cline{2-11}
 & \multirow[t]{3}{*}{10,000} & 5 & 22.06 & 0.234 & 0.76 & 0.73 & 0.08 & 36 & 0.8 & 0.01 \\
 &  & 10 & 21.27 & 0.263 & 0.81 & 0.88 & 0.05 & 38 & 0.5 & 0.01 \\
 &  & 20 & 21.44 & 0.263 & 0.81 & 0.88 & 0.05 & 29 & 0.5 & 0.01 \\
\cline{1-11} \cline{2-11}
\bottomrule
\end{tabular}

    \caption{Shutterstock: Tabular Results}
    \label{fig:shutterstock_tab_results}
\end{figure}

\end{document}